\documentclass{article}

\usepackage[letterpaper,margin=1in]{geometry}

\usepackage{datetime}
\usepackage[utf8]{inputenc} 
\usepackage[T1]{fontenc}    
\usepackage{url}            
\usepackage{booktabs}       
\usepackage{amsfonts}       
\usepackage{nicefrac}       
\usepackage{microtype}      
\usepackage{color}
\usepackage{multirow}

\usepackage{amsmath,amsthm}
\usepackage{graphicx}
\usepackage{caption}
\usepackage{subcaption}
\usepackage{bbm}
\usepackage{cite}
\usepackage{comment}
\usepackage{listings}

\usepackage{tikz}
\usetikzlibrary{shapes,backgrounds}
\usetikzlibrary{arrows}

\newtheorem{theorem}{Theorem}

\theoremstyle{definition}

\newtheorem{example}{Example}[section]

\def\compactify{\itemsep=0pt \topsep=0pt \partopsep=0pt \parsep=0pt}

\usepackage{breqn}

\DeclareMathOperator{\artanh}{artanh}
\DeclareMathOperator{\sech}{sech}

\usepackage{thmtools}
\usepackage{thm-restate}

\usepackage{algorithm}
\usepackage{algorithmic}

\newcommand{\R}{\mathbb{R}}

\newcommand{\Abs}[1]{\left\lvert #1 \right\rvert}
\newcommand{\Prob}[2][]{\mathbf{P}_{#1}\left( {#2} \right) }
\newcommand{\Probc}[3][]{\mathbf{P}_{#1}\left( {#2} \middle | {#3} \right) }
\newcommand{\Exv}[2][]{\mathbf{E}_{#1}\left[ #2 \right]}
\newcommand{\Exvc}[3][]{\mathbf{E}_{#1}\left[ #2 \middle | {#3} \right]}
\newcommand{\Var}[2][]{\mathbf{Var}_{#1}\left( #2 \right)}
\newcommand{\Varc}[3][]{\mathbf{Var}_{#1}\left( #2 \middle | {#3} \right)}
\newcommand{\Cov}[2][]{\mathbf{Cov}_{#1}\left( #2 \right)}
\newcommand{\Covc}[3][]{\mathbf{Cov}_{#1}\left( #2 \middle | {#3} \right)}

\newcommand{\F}{\mathcal{F}}

\newcommand{\norm}[1]{\left\| #1 \right\|}
\usepackage{txfonts}

\newdate{date}{26}{05}{2016}
\date{\displaydate{date}}

\title{Data Programming:\\Creating Large Training Sets, Quickly}

\newcommand{\finaledit}[1]{#1}

\author{
  Alexander Ratner, Christopher De Sa, Sen Wu, Daniel Selsam, Christopher R\'{e}\\
  Stanford University\\
  \texttt{\{ajratner,cdesa,senwu,dselsam,chrismre\}@stanford.edu} \\
}

\begin{document}
\lstset{
    language=Python,
    basicstyle=\tiny,
    escapeinside={(*@}{@*)}}
    

\maketitle

\begin{abstract}
Large labeled training sets are the critical building blocks of supervised learning methods
and are key enablers of deep learning techniques. For some applications, creating labeled 
training sets is the most time-consuming and expensive part of applying machine learning. 
We therefore propose a paradigm for the programmatic creation of training sets called 
\textit{data programming} in which users express weak supervision strategies or domain heuristics 
as \textit{labeling functions}, which are programs that label subsets of the data,
but that are noisy and may conflict. We show that by explicitly representing this training set
labeling process as a generative model, we can ``denoise'' the generated training set, and 
establish theoretically that we can recover the parameters of these generative models in a
handful of settings. We then show how to modify a discriminative loss function to make it 
noise-aware, and demonstrate our method over a range of discriminative models including 
logistic regression and LSTMs. Experimentally, on the 2014 TAC-KBP Slot Filling challenge, 
we show that data programming would have led to a new winning score, and also show that 
applying data programming to an LSTM model leads to a TAC-KBP score almost 6 F1 points over 
a state-of-the-art LSTM baseline (and into second place in the competition). Additionally, 
in initial user studies we observed that data programming may be an easier way for non-experts 
to create machine learning models when training data is limited or unavailable.
\end{abstract}

\section{Introduction}
\label{secIntro}

Many of the major machine learning breakthroughs of the last
decade have been catalyzed by the release of a new labeled
training
dataset.\footnote{\tiny{\url{http://www.spacemachine.net/views/2016/3/datasets-over-algorithms}}}
Supervised learning approaches that use such datasets have
increasingly become key building blocks of applications throughout
science and industry.  This trend has also been fueled by the
recent empirical success of automated feature generation approaches,
notably deep learning methods such as long short term memory (LSTM)
networks~\cite{hochreiter1997long}, which ameliorate the burden of feature
engineering given large enough labeled training sets. For many
real-world applications, however, large hand-labeled training sets do
not exist, and are prohibitively expensive to create due to
requirements that labelers be experts in the application domain.
Furthermore, applications' needs often change, necessitating new or
modified training sets.

To help reduce the cost of training set creation, we propose
\textit{data programming}, a paradigm for the programmatic creation and modeling
of training datasets. Data programming provides a simple, unifying framework for
\textit{weak supervision}, in which training labels are noisy and may be from multiple, 
potentially overlapping sources. In data programming, users encode this weak supervision 
in the form of \textit{labeling functions}, which are user-defined programs that each
provide a label for some subset of the data, and collectively generate a large but 
potentially overlapping set of training labels. Many different weak supervision approaches 
can be expressed as labeling functions, such as strategies which utilize existing
knowledge bases (as in distant supervision~\cite{Mintz2009}), 
model many individual annotator's labels (as in crowdsourcing), or leverage a 
combination of domain-specific patterns and dictionaries. Because of this, labeling functions
may have widely varying error rates and may conflict on certain data points.
To address this, we model the labeling functions as a generative process, which lets us
automatically denoise the resulting training set by learning the
accuracies of the labeling functions along with their correlation structure.  In
turn, we use this model of the training set to optimize a stochastic
version of the loss function of the discriminative model that we
desire to train. 
We show that, given certain conditions on the labeling functions,
our method achieves the same asymptotic
scaling as supervised learning methods, but that our scaling depends on the
amount of \emph{unlabeled} data, and uses only a fixed number of labeling functions.

Data programming is in part motivated by the challenges that users
faced when applying prior programmatic supervision approaches, and is
intended to be a new software engineering paradigm for the creation and
management of training sets.
For example, consider the scenario when two labeling functions of
differing quality and scope overlap and possibly conflict on certain training examples; in prior approaches
the user would have to decide which one to use, or how to somehow
integrate the signal from both.  In data programming, we accomplish this
automatically by learning a model of the training set that includes both labeling functions. Additionally,
users are often aware of, or able to induce, dependencies between their labeling
functions.  In data programming, users can provide a dependency graph
to indicate, for example, that two labeling functions are similar, or
that one ``fixes'' or ``reinforces'' another.  
We describe cases in
which we can learn the strength of these dependencies, and for which
our generalization is again asymptotically identical to the supervised case.

One further motivation for our method is driven by the observation
that users often struggle with selecting \textit{features} 
for their models, which is a traditional development bottleneck given fixed-size training sets.
However, initial feedback from users
suggests that writing labeling functions in the framework of data
programming may be easier\finaledit{~\cite{ehrenberg2016data}}.
While the impact of a feature on end performance is
dependent on the training set and on statistical
characteristics of the model, a labeling function has a simple
and intuitive optimality criterion: that it labels data correctly.
Motivated by this, we explore whether we can flip the traditional
machine learning development process on its head, having users instead focus on 
generating training sets large enough to support automatically-generated features.

\paragraph*{Summary of Contributions and Outline}
Our first contribution is the \textit{data programming} framework,
in which users can implicitly describe a rich generative model for a training set
in a more flexible and general way than in previous approaches. In
Section~\ref{secDataProgramming}, we first explore a simple model in
which labeling functions are conditionally independent. We show here
that under certain conditions, the sample complexity is nearly the same as in
the labeled case. In
Section~\ref{sec:dep}, we extend our results to more sophisticated
data programming models, generalizing
related results in crowdsourcing~\cite{karger2011iterative}.  In
Section~\ref{secExperiments}, we validate our approach experimentally
on large real-world text relation extraction tasks in genomics,
pharmacogenomics and news domains, where
we show an average 2.34 point F1 score 
improvement over a baseline \finaledit{distant} supervision approach---including
what would have been a new
competition-winning score for the 2014 TAC-KBP Slot Filling competition.
Using LSTM-generated features, we \finaledit{additionally} would have placed second in
this competition, achieving a 5.98 point F1 score gain
over a state-of-the-art LSTM baseline~\cite{verga2015multilingual}.
Additionally, we describe promising feedback from a 
usability study with a group of bioinformatics users.

\section{Related Work}
\label{secRelatedWork}

Our work builds on many previous approaches in machine learning.
\textit{Distant supervision} is one approach for programmatically
creating training sets.  The canonical example
is relation extraction from text, wherein a knowledge base of
known relations is heuristically mapped to
an input corpus~\cite{craven1999constructing,Mintz2009}.  Basic extensions
group examples by surrounding textual patterns, and cast the problem as a \textit{multiple
  instance learning}
one~\cite{riedel2010modeling,hoffmann2011knowledge}. Other extensions
model the accuracy of these surrounding textual patterns using a discriminative
feature-based model~\cite{roth2013feature}, or generative models such
as hierarchical topic
models~\cite{alfonseca2012pattern,roth2013combining,takamatsu2012reducing}.
Like our approach, these latter methods model a generative process of
training set creation, however in a proscribed way that is not based
on user input as in our approach. There is also a wealth of examples where
additional heuristic patterns used to label training data are
collected from unlabeled data~\cite{bunescu2007learning} or directly
from users~\cite{shin2015incremental,mallory2015large}, in a similar manner to
our approach, but without any framework to deal with the fact that said
labels are explicitly noisy.

\textit{Crowdsourcing} is widely used for various machine learning
tasks~\cite{krishna2016visual,gao2011harnessing}.  Of particular
relevance to our problem setting is the theoretical question of how to
model the accuracy of various experts without ground truth
available, classically raised
in the context of crowdsourcing~\cite{dawid1979maximum}.  More recent
results provide formal guarantees even in the absence of
labeled data using various approaches~\cite{karger2011iterative,parisi2014ranking,NIPS2014_5253,NIPS2014_5431,Dalvi:2013:ACB:2488388.2488414,joglekar2015comprehensive}.
Our model can capture the
\finaledit{basic model of the crowdsourcing setting, and can be considered equivalent}
in the independent case (Sec.~\ref{secDataProgramming}).
However, in addition to generalizing beyond getting inputs solely from human annotators,
we also model user-supplied dependencies between the
``labelers'' in our model, which is not natural within the context
of crowdsourcing.  Additionally, while crowdsourcing results focus on the regime
of a large number of labelers each labeling a small subset of
the data, we consider a small set of labeling functions
each labeling a large portion of the dataset.

\textit{Co-training} is a classic procedure for effectively utilizing both a small
amount of labeled data and a large amount of unlabeled data by selecting two
conditionally independent \textit{views} of the data~\cite{blum1998combining}.
In addition to not needing a set of labeled data, and allowing for more than two
views (labeling functions in our case), our approach allows explicit modeling of
dependencies between views, for example allowing observed issues with dependencies between
views to be explicitly modeled~\cite{krogel2004multi}.

\textit{Boosting} is a well known procedure for combining the output of
many ``weak'' classifiers to create a strong classifier in a supervised setting~\cite{schapire2012boosting}.
Recently, boosting-like methods have been proposed which leverage unlabeled data
in addition to labeled data, which is also used to set constraints
on the accuracies of the individual classifiers being ensembled~\cite{balsubramani2015scalable}.
This is similar in spirit to our approach, except that
labeled data is not explicitly necessary in ours, and richer dependency structures
between our ``heuristic'' classifiers (labeling functions) are supported.

The general case of \emph{learning with noisy labels} is treated both in classical~\cite{LUGOSI199279}
and more recent contexts~\cite{NIPS2013_5073}.  It has also been studied specifically
in the context of \textit{label-noise robust} logistic regression~\cite{bootkrajang2012label}.
We consider the more general scenario where multiple noisy labeling
functions can conflict and have dependencies.

\section{The Data Programming Paradigm}
\label{secDataProgramming}

In many applications, we would like to use machine learning, but we
face the following challenges: (i) \textit{hand-labeled} training data
is not available, and is prohibitively expensive to obtain in
sufficient quantities as it requires expensive domain expert \finaledit{labelers}; (ii)
\textit{related external knowledge bases} are either unavailable or
insufficiently specific, precluding a traditional distant supervision
or co-training approach; (iii) \textit{application specifications} are
in flux, changing the model we ultimately wish to learn.

In such a setting, we would like a simple, scalable and adaptable
approach for supervising a model applicable to our problem.  More
specifically, we would ideally like our approach to achieve
$\epsilon$ expected loss with high probability, given $O(1)$
\textit{inputs} of some sort from a domain-expert user, rather than
the traditional $\tilde{O}(\epsilon^{-2})$ hand-labeled training
examples required by most supervised methods (where $\tilde O$ notation hides
logarithmic factors).  To this end, we propose
\textit{data programming}, a paradigm for the
programmatic creation of training sets, which enables domain-experts to
more rapidly train machine learning systems and has the potential for this
type of scaling of expected loss.
In data programming, rather than manually
labeling each example, users instead describe the \textit{processes by
which} these points could be labeled by providing a set of heuristic
rules called \textit{labeling functions}.

In the remainder of this paper, we focus on a binary
classification task in which we have a distribution $\pi$ over object and class pairs
$(x,y) \in \mathcal{X} \times \{-1,1\}$, and we are concerned with minimizing the logistic loss under a 
linear model given some \emph{features},
\[
 l(w)
 =
 \Exv[(x,y) \sim \pi]{\log(1 + \exp(-w^T f(x) y))},
\] 
where without loss of generality, we assume that $\norm{f(x)} \le 1$.
Then, a labeling function $\lambda_i: \mathcal{X} \mapsto \{-1,0,1\}$
is a user-defined function that encodes some domain heuristic, which
provides a (non-zero) label for some subset of the objects.  As part
of a \textit{data programming specification}, a user provides some $m$
labeling functions, which we denote in vectorized form as $\lambda:
\mathcal{X} \mapsto \{-1,0,1\}^m$.

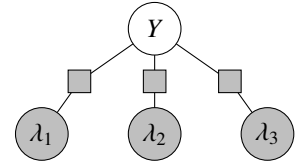
\begin{figure}
\centering
\begin{subfigure}{.7\textwidth}
\begin{lstlisting}
def lambda_1(x):
    return 1 if (x.gene,x.pheno) in KNOWN_RELATIONS_1 else 0

def lambda_2(x):
    return (*@-@*)1 if re.match(r'.*not cause.*', x.text_between) else 0

def lambda_3(x):
    return 1 if re.match(r'.*associated.*', x.text_between) 
            and (x.gene,x.pheno) in KNOWN_RELATIONS_2 else 0
\end{lstlisting}
\caption{An example set of three labeling functions written by a user.}
\label{ind-example-lfs}
\end{subfigure}%
\begin{subfigure}{.3\textwidth}
\centering
\begin{tikzpicture}[every node/.style={inner sep=0,outer sep=0}]

\draw (0,1.4) node[draw,circle,fill=white,minimum size=0.7cm](y) {$Y$};

\draw (-1,0.7) node[draw,fill=lightgray,rectangle,minimum size=0.3cm](l1) {};
\draw (0,0.7) node[draw,fill=lightgray,rectangle,minimum size=0.3cm](l2) {};
\draw (1,0.7) node[draw,fill=lightgray,rectangle,minimum size=0.3cm](l3) {};
\draw (y) -- (l1);
\draw (y) -- (l2);
\draw (y) -- (l3);

\draw (-1.5,0) node[draw,fill=lightgray,circle,minimum size=0.7cm](l1x) {$\lambda_1$};
\draw (0,0) node[draw,fill=lightgray,circle,minimum size=0.7cm](l2x) {$\lambda_2$};
\draw (1.5,0) node[draw,fill=lightgray,circle,minimum size=0.7cm](l3x) {$\lambda_3$};
\draw (l1) -- (l1x);
\draw (l2) -- (l2x);
\draw (l3) -- (l3x);

\end{tikzpicture}
\caption{The generative model of a training set defined by the user input (unary factors omitted).}
\label{ind-example-fg}
\end{subfigure}

\caption{An example
of extracting mentions of gene-disease relations from the scientific
literature.}
\label{ind-example}
\end{figure}

\begin{example}
\label{ex1}
To gain intuition about labeling functions, we describe a simple text relation extraction example.
In Figure ~\ref{ind-example}, we consider the task of classifying co-occurring gene and disease 
mentions as either expressing a causal relation or not.  For example, 
given the sentence ``Gene A causes disease B'', the object $x = (A,B)$ has true class $y = 1$.  
To construct a training set, the user writes three labeling functions 
(Figure~\ref{ind-example-lfs}). In $\lambda_1$, an external structured knowledge base is 
used to label a few objects with relatively high accuracy, and is 
equivalent to a traditional distant supervision rule (see Sec.~\ref{secRelatedWork}). 
$\lambda_2$ uses a purely heuristic approach to label a much larger number of examples 
with lower accuracy.  Finally, $\lambda_3$ is a ``hybrid'' labeling function, which leverages a knowledge base and a heuristic.
\end{example}

A labeling function need not have perfect accuracy or recall; rather, it represents a pattern that the user 
wishes to impart to their model and that is easier to encode as a labeling function than as a set of 
hand-labeled examples.  As illustrated in Ex.~\ref{ex1}, labeling functions can be based on external 
knowledge bases, libraries or ontologies, can express heuristic patterns, or some hybrid of these
types; we see evidence for the existence of such diversity in our experiments (Section~\ref{secExperiments}). The use of 
labeling functions is also strictly more general than manual annotations, as a manual annotation can 
always be directly encoded by a labeling function.  Importantly, labeling functions can
overlap, conflict, and even have dependencies which users can provide as part of the data programming 
specification (see Section~\ref{sec:dep}); our approach provides a simple framework for these inputs.

\paragraph*{Independent Labeling Functions}
We first describe a model in which the
labeling functions label independently, given the true label class.
Under this model, each labeling function $\lambda_i$ has some probability
$\beta_i$ of labeling an object and then some probability $\alpha_i$ of
labeling the object correctly; for simplicity we also assume here that
each class has probability 0.5. This
model has distribution
\begin{equation}
  \label{eqnSimpleGenerativeModel}
  \mu_{\alpha, \beta}(\Lambda, Y)
  =
  \frac{1}{2}
  \prod_{i=1}^m
  \left(
    \beta_i \alpha_i \mathbf{1}_{\{ \Lambda_i = Y \}} 
    +
    \beta_i (1 - \alpha_i) \mathbf{1}_{\{ \Lambda_i = -Y \}} 
    +
    (1 - \beta_i) \mathbf{1}_{\{ \Lambda_i = 0 \}} 
  \right),
\end{equation}
where $\Lambda \in \{-1,0,1\}^m$ contains the labels output by the labeling
functions, and $Y \in \{-1,1\}$ is the predicted class.  If we allow the
parameters $\alpha \in \R^m$ and $\beta \in \R^m$ to vary,
(\ref{eqnSimpleGenerativeModel}) specifies a
family of generative models.
In order to expose the scaling of the expected loss as the size of
the unlabeled dataset changes, we will assume here that
$0.3 \le \beta_i \le 0.5$ and $0.8 \le \alpha_i \le 0.9$.
We note that while these arbitrary constraints can be changed, they are roughly
consistent with our applied experience, where users tend to write high-accuracy
and high-coverage labeling functions.

Our first goal will be to learn which parameters $(\alpha, \beta)$ are
most consistent with our observations---our unlabeled training set---using
maximum likelihood estimation.
To do this for a
particular training set $S \subset \mathcal{X}$, we will solve the problem
\begin{equation}
  \label{eqnParameterLearning}
  (\hat \alpha, \hat \beta)
  =
  \arg \max_{\alpha, \beta}
  \sum_{x \in S}
  \log \Prob[(\Lambda, Y) \sim \mu_{\alpha, \beta}]{
    \Lambda = \lambda(x)
  }
  =
  \arg \max_{\alpha, \beta}
  \sum_{x \in S}
  \log \left(
    \sum_{y'\in\{-1,1\}}
    \mu_{\alpha, \beta}(\lambda(x),y')
    \right)
\end{equation}
In other words, we are maximizing the probability that the observed
labels produced on our training examples occur under the generative
model in (\ref{eqnSimpleGenerativeModel}).  In our experiments, we use
stochastic gradient descent to solve this problem; since this is a
standard technique, we defer its analysis to the appendix.

\paragraph*{Noise-Aware Empirical Loss}
Given that our parameter learning phase has successfully found some
$\hat \alpha$ and $\hat \beta$ that
accurately describe the training set,
we can now proceed to estimate the parameter $w$ which
minimizes the
expected risk of a linear model over our feature mapping $f$, given $\hat \alpha,\hat \beta$.
To do so, we define the \textit{noise-aware empirical risk}
$L_{\hat \alpha, \hat \beta}$ with regularization parameter $\rho$,
and compute the \textit{noise-aware empirical risk minimizer}
\begin{equation}
  \label{eqnClassifierLearning}
  \hat w
  =
  \arg \min_{w} L_{\hat \alpha, \hat \beta}(w; S)
  = 
  \arg \min_{w}
  \frac{1}{\Abs{S}} \sum_{x \in S}
  \Exvc[(\Lambda, Y) \sim \mu_{\hat \alpha, \hat \beta}]{
    \log\left(1 + e^{-w^T f(x) Y}\right)
  }{\Lambda = \lambda(x)}
  +
  \rho \norm{w}^2
\end{equation}
This is a logistic
regression problem, so it can be solved using stochastic gradient
descent as well.

We can in fact prove that
stochastic gradient descent
running on (\ref{eqnParameterLearning}) and
(\ref{eqnClassifierLearning}) is guaranteed to produce accurate
estimates, under conditions which we describe now. First, the problem distribution
$\pi$ needs to be accurately modeled by some distribution $\mu$ in the
family that we are trying to learn.  That is, for some $\alpha^*$ and $\beta^*$,
\begin{equation}
  \label{eqnDPA1}
  \forall \Lambda \in \{-1,0,1\}^m, Y \in \{-1,1\}, \,
  \Prob[(x,y) \sim \pi^*]{\lambda(x) = \Lambda, \, y = Y}
  =
  \mu_{\alpha^*, \beta^*}(\Lambda, Y).
\end{equation}
Second, given an example $(x, y) \sim \pi^*$,
the class label $y$ must be independent of the features $f(x)$ given the labels
$\lambda(x)$.  That is,
\begin{equation}
  \label{eqnDPA2}
  (x, y) \sim \pi^*
  \Rightarrow
  y \perp f(x)\ |\ \lambda(x).
\end{equation}
This assumption encodes the idea that the labeling functions, while they
may be arbitrarily dependent on the features, provide
sufficient information to accurately identify the class.
Third, we assume that the
algorithm used to solve (\ref{eqnClassifierLearning}) has bounded
generalization risk such that for some parameter $\chi$,
\begin{equation}
  \label{eqnDPA3}
  \Exv[\hat w]{
    \Exv[S]{L_{\hat \alpha, \hat \beta}(\hat w; S)}
    - 
    \min_w
    \Exv[S]{L_{\hat \alpha, \hat \beta}(w; S)}
  } \le \chi.
\end{equation}
Under these conditions,
we make the following statement about the accuracy of our estimates,
which is a simplified version of a theorem that is detailed in the appendix.

\begin{theorem}
  \label{stmtIndepDataProgramming}
  Suppose that we run data programming, solving the problems in
  (\ref{eqnParameterLearning}) and (\ref{eqnClassifierLearning}) using
  stochastic gradient descent
  to produce $(\hat \alpha, \hat \beta)$ and $\hat w$.  Suppose further that
  our setup satisfies the conditions (\ref{eqnDPA1}),
  (\ref{eqnDPA2}), and (\ref{eqnDPA3}), and suppose that $m \ge 2000$. Then
  for any $\epsilon > 0$, if the number of labeling functions $m$ and
  the size of the input dataset $S$ are large enough that
    $$\Abs{S} \ge
    \frac{356}{\epsilon^2}
    \log\left(
      \frac{m}{3 \epsilon}
    \right)$$
  then our expected parameter error and generalization risk can be
  bounded by 
  \begin{align*}
    \Exv{\norm{\hat \alpha - \alpha^*}^2} &\le m \epsilon^2
    &
    \Exv{\norm{\hat \beta - \beta^*}^2} &\le m \epsilon^2
    &
    \Exv{l(\hat w) - \min_w l(w)} &\le \chi + \frac{\epsilon}{27 \rho}.
  \end{align*}
\end{theorem}
We select $m \geq 2000$ to simplify the
statement of the theorem and give the reader a feel for how
$\epsilon$ scales with respect to $\Abs{S}$. The full theorem with scaling
in each parameter (and for arbitrary $m$) is presented in the appendix.
This result establishes that to achieve both expected loss and
parameter estimate error $\epsilon$, it suffices
to have only $m = O(1)$ labeling functions and $\Abs{S} = \tilde
O(\epsilon^{-2})$ training examples, which is the same asymptotic
scaling exhibited by methods that use labeled data.  This means that
data programming achieves the same learning rate as methods that use
labeled data, while requiring asymptotically less work from its users,
who need to specify $O(1)$ labeling functions rather than manually
label $\tilde O(\epsilon^{-2})$ examples. In contrast, in the
crowdsourcing setting~\cite{karger2011iterative}, the number of 
workers $m$ tends to
infinity while here it is constant while the dataset
grows. These results provide some
explanation of why our experimental results suggest that a small number
of rules with a large unlabeled training set can be effective at even
complex natural language processing tasks.

\section{Handling Dependencies}
\label{sec:dep}

\begin{figure}
\centering
\begin{subfigure}{.4\textwidth}
\begin{tikzpicture}[every node/.style={inner sep=0,outer sep=0}]

\draw (0,1.6) node[draw,circle,fill=white,minimum size=0.7cm](y) {$Y$};

\draw (-0.8,0.8) node[draw,fill=lightgray,rectangle,minimum size=0.3cm](l1) {};
\draw (0.8,0.8) node[draw,fill=lightgray,rectangle,minimum size=0.3cm](l2) {};
\draw (y) -- (l1);
\draw (y) -- (l2);

\draw (-1.2,0) node[draw,fill=lightgray,circle,minimum size=0.7cm](l1x) {$\lambda_1$};
\draw (1.2,0) node[draw,fill=lightgray,circle,minimum size=0.7cm](l2x) {$\lambda_2$};
\draw (l1) -- (l1x);
\draw (l2) -- (l2x);

\draw (0,0) node[draw,fill=lightgray,rectangle,minimum size=0.5cm](sx) {$s$};
\draw (l1x) -- (sx);
\draw (l2x) -- (sx);

\end{tikzpicture}
\begin{lstlisting}
lambda_1(x) = f(x.word)
lambda_2(x) = f(x.lemma)

Similar(lambda_1, lambda_2)
\end{lstlisting}
\label{dep-ex-1}
\end{subfigure}%
\begin{subfigure}{.4\textwidth}
\begin{tikzpicture}[every node/.style={inner sep=0,outer sep=0}]

\draw (0,1.2) node[draw,circle,fill=white,minimum size=0.7cm](y) {$Y$};

\draw (-1.5,-0.3) node[draw,fill=lightgray,circle,minimum size=0.7cm](l2x) {$\lambda_2$};
\draw (0,-0.15) node[draw,fill=lightgray,circle,minimum size=0.7cm](l1x) {$\lambda_1$};
\draw (1.5,-0.3) node[draw,fill=lightgray,circle,minimum size=0.7cm](l3x) {$\lambda_3$};

\draw (0,0.6) node[draw,fill=lightgray,rectangle,minimum size=0.3cm](l1) {};
\draw (y) -- (l1);
\draw (l1x) -- (l1);

\draw (-0.9,0.6) node[draw,fill=lightgray,rectangle,minimum size=0.5cm](f12) {$f$};
\draw (0.9,0.6) node[draw,fill=lightgray,rectangle,minimum size=0.5cm](r13) {$r$};
\draw (f12) -- (l2x);
\draw (f12) -- (l1x);
\draw (f12) -- (y);
\draw (r13) -- (l3x);
\draw (r13) -- (l1x);
\draw (r13) -- (y);

\end{tikzpicture}
\begin{lstlisting}
lambda_1(x) = f('.*cause.*')
lambda_2(x) = f('.*not cause.*')
lambda_3(x) = f('.*cause.*')

Fixes(lambda_1, lambda_2)
Reinforces(lambda_1, lambda_3)
\end{lstlisting}
\label{dep-ex-2}
\end{subfigure}%
\begin{subfigure}{.2\textwidth}
\begin{tikzpicture}[every node/.style={inner sep=0,outer sep=0}]

\draw (0,1.6) node[draw,circle,fill=white,minimum size=0.7cm](y) {$Y$};

\draw (-0.8,0.8) node[draw,fill=lightgray,rectangle,minimum size=0.3cm](l1) {};
\draw (0.8,0.8) node[draw,fill=lightgray,rectangle,minimum size=0.3cm](l2) {};
\draw (y) -- (l1);
\draw (y) -- (l2);

\draw (-1.2,0) node[draw,fill=lightgray,circle,minimum size=0.7cm](l1x) {$\lambda_1$};
\draw (1.2,0) node[draw,fill=lightgray,circle,minimum size=0.7cm](l2x) {$\lambda_2$};
\draw (l1) -- (l1x);
\draw (l2) -- (l2x);

\draw (0,0) node[draw,fill=lightgray,rectangle,minimum size=0.5cm](ex) {$e$};
\draw (l1x) -- (ex);
\draw (l2x) -- (ex);

\end{tikzpicture}
\begin{lstlisting}
lambda_1(x) = x in DISEASES_A
lambda_2(x) = x in DISEASES_B

Excludes(lambda_1, lambda_2)
\end{lstlisting}
\label{dep-ex-3}
\end{subfigure}

\caption{Examples of labeling function dependency predicates.}
\label{dep-example}
\end{figure}
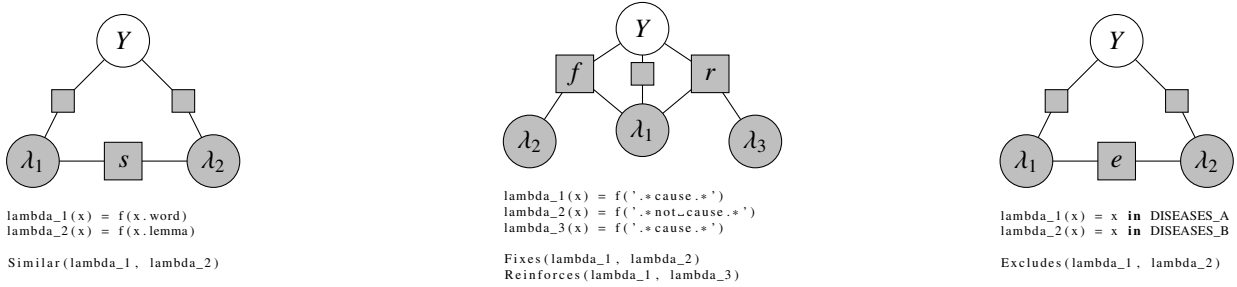

In our experience with data programming, we have found that users
often write labeling functions that have clear dependencies among
them. As more labeling functions are added as the system is developed,
an implicit dependency structure arises naturally amongst the labeling
functions: modeling these dependencies can in some cases improve
accuracy. We describe a method by which the user can specify this dependency
knowledge as a \emph{dependency graph}, and show how the system can
use it to produce better parameter estimates.

\paragraph*{Label Function Dependency Graph}
To support the injection of dependency information into the model, we
augment the data programming specification with 
a \textit{label function dependency graph},
$G \subset \mathcal{D} \times \{1, \ldots, m\} \times \{1, \ldots, m\}$,
which is a directed graph over the labeling functions, each of the edges of
which is associated with a \emph{dependency type} from a class of dependencies
$\mathcal{D}$ appropriate to the domain.  
From our experience with practitioners, we identified four commonly-occurring
types of dependencies as illustrative examples:
\emph{similar}, \emph{fixing}, \emph{reinforcing}, and
\emph{exclusive} (see Figure \ref{dep-example}).

For example, suppose that we have two functions $\lambda_1$ and $\lambda_2$,
and $\lambda_2$
typically labels only when (i) $\lambda_1$ also labels, (ii) $\lambda_1$
and $\lambda_2$ disagree in their labeling, and (iii) $\lambda_2$ is actually
correct.  We call this a fixing dependency, since $\lambda_2$
\emph{fixes} mistakes made by $\lambda_1$.  If 
$\lambda_1$ and $\lambda_2$ were to typically agree rather than disagree,
this would be a reinforcing dependency, since $\lambda_2$ reinforces
a subset of the labels of $\lambda_1$.

\paragraph*{Modeling Dependencies}
The presence of dependency information means that we can no longer model our
labels using the simple Bayesian network in (\ref{eqnSimpleGenerativeModel}).
Instead, we model our distribution as a factor graph.
This standard technique lets us describe the family of generative distributions
in terms of a known \emph{factor function}
$h: \{-1,0,1\}^m \times \{-1,1\} \mapsto \{-1,0,1\}^M$ (in which each
entry $h_i$ represents a factor), 
and an unknown parameter $\theta \in \R^M$ as
\[
  \mu_{\theta}(\Lambda, Y)
  =
  Z_{\theta}^{-1} \exp(
    \theta^T h(\Lambda, Y)
  ),
\]
where $Z_{\theta}$ is the \emph{partition function} which ensures that $\mu$
is a distribution.  Next, we will describe how we define $h$ using
information from the dependency graph.

To construct $h$, we will start with some base factors, which we
inherit from (\ref{eqnSimpleGenerativeModel}),
and then augment them with additional factors representing dependencies.
For all $i \in \{1, \ldots, m\}$, we let
\[
  h_0(\Lambda, Y) = Y,
  \quad
  h_i(\Lambda, Y) = \Lambda_i Y,
  \quad
  h_{m+i}(\Lambda, Y) = \Lambda_i,
  \quad
  h_{2m+i}(\Lambda, Y) = \Lambda_i^2 Y,
  \quad
  h_{3m+i}(\Lambda, Y) = \Lambda_i^2.
\]
These factors alone are sufficient to describe any distribution for which
the labels are mutually independent, given the class: this includes the
independent family in (\ref{eqnSimpleGenerativeModel}).

We now proceed by adding additional factors to $h$, which model the 
dependencies encoded in $G$.  For each dependency edge $(d, i, j)$, we 
add one or more factors to $h$ as follows.
For a near-duplicate dependency on $(i,j)$, we add a single factor
$h_{\iota}(\Lambda, Y) = \mathbf{1}\{ \Lambda_i = \Lambda_j \}$,
which increases our
prior probability that the labels will agree.  For a fixing dependency, we 
add two factors,
$h_{\iota}(\Lambda, Y) = -\mathbf{1}\{ \Lambda_i = 0 \wedge \Lambda_j \ne 0 \}$
and
$h_{\iota+1}(\Lambda, Y) = \mathbf{1}\{ \Lambda_i = -Y \wedge \Lambda_j = Y \}$,
which encode the idea that $\lambda_j$ labels only when $\lambda_i$ does,
and that $\lambda_j$ fixes errors made by $\lambda_i$.  The factors
for a reinforcing dependency are the same, except that
$h_{\iota+1}(\Lambda, Y) = \mathbf{1}\{ \Lambda_i = Y \wedge \Lambda_j = Y \}$.
Finally, for an exclusive dependency, we have a single factor
$h_{\iota}(\Lambda, Y) = -\mathbf{1}\{ \Lambda_i \ne 0 \wedge \Lambda_j \ne 0 \}$.

\paragraph{Learning with Dependencies}
We can again solve a maximum likelihood problem like
(\ref{eqnParameterLearning}) to learn the parameter $\hat \theta$.
Using the results, we can
continue on to find the noise-aware empirical loss minimizer by
solving the problem in (\ref{eqnClassifierLearning}).  In order to
solve these problems in the dependent case, we typically invoke
stochastic gradient descent, using Gibbs sampling to sample from the
distributions used in the gradient update. Under conditions similar to
those in Section \ref{secDataProgramming}, we can again provide a bound on
the accuracy of these results.  We define these conditions now.  First,
there must be some set $\Theta \subset \R^M$ that we know our 
parameter lies in.  This is analogous to the assumptions on $\alpha_i$ and
$\beta_i$ we made in Section \ref{secDataProgramming}, and we can state
the following analogue of (\ref{eqnDPA1}):
\begin{equation}
  \label{eqnDepDPA1}
  \exists \theta^* \in \Theta \text{ s.t. }
  \forall (\Lambda, Y) \in \{-1,0,1\}^m \times \{-1,1\}, \:
  \Prob[(x,y) \sim \pi^*]{\lambda(x) = \Lambda, \, y = Y}
  =
  \mu_{\theta^*}(\Lambda, Y).
\end{equation}
Second, for any $\theta \in \Theta$, it must be possible to accurately
learn $\theta$ from full (i.e. labeled) samples of $\mu_{\theta}$.
More specifically, there exists an unbiased estimator $\hat \theta(T)$
that is a function of some dataset $T$ of independent samples from
$\mu_{\theta}$ 
such that, for some $c > 0$ and for all $\theta \in \Theta$,
\begin{equation}
  \label{eqnDepDPA2}
  \Cov{\hat \theta(T)} \preceq (2 c \Abs{T})^{-1} I.
\end{equation}
Third, for any two feasible models $\theta_1$ and $\theta_2 \in \Theta$,
\begin{equation}
  \label{eqnDepDPA3}
  \Exv[(\Lambda_1, Y_1) \sim \mu_{\theta_1}]{
    \Varc[(\Lambda_2, Y_2) \sim \mu_{\theta_2}]{Y_2}{\Lambda_1 = \Lambda_2}
  }
  \le
  c M^{-1}.
\end{equation}
That is, we'll usually be reasonably sure in our guess for the value of
$Y$, even if we guess using distribution $\mu_{\theta_2}$ while the
the labeling functions were actually sampled from (the possibly totally
different) $\mu_{\theta_1}$.
We can now prove the following result about the accuracy
of our estimates.

\begin{theorem}
  \label{thmDependentDataProgramming}
  Suppose that we run stochastic gradient descent to produce $\hat
  \theta$ and $\hat w$, and that our setup satisfies the conditions
  (\ref{eqnDPA2})-(\ref{eqnDepDPA3}). Then for any $\epsilon > 0$, if the
  input dataset $S$ is large enough that
  \begin{align*}
    |S|
    &\ge
    \frac{2}{c^2 \epsilon^2}
    \log\left(
      \frac{2 \norm{\theta_0 - \theta^*}^2}{\epsilon}
    \right),
  \end{align*}
  then our expected parameter error and generalization risk can be
  bounded by 
  \begin{align*}
    \Exv{\norm{\hat \theta - \theta^*}^2} &\le M \epsilon^2
    &
    \Exv{l(\hat w) - \min_w l(w)} &\le \chi + \frac{c \epsilon}{2 \rho}.
  \end{align*}
\end{theorem}

As in the independent case, this shows that we need only $\Abs{S} =
\tilde O(\epsilon^{-2})$ unlabeled training examples to achieve error
$O(\epsilon)$, which is the same asymptotic scaling as
supervised learning methods. This suggests that while we pay a
computational penalty for richer dependency structures, we are no less
statistically efficient.  In the appendix, we provide more details,
including an explicit description of the algorithm and the step size used
to achieve this result.

\section{Experiments}
\label{secExperiments}

\begin{table}
\center
\small
\begin{tabular}{ | l | l || c | c | c || c | c | c || c | c | c | }
\hline
 & & \multicolumn{3}{|c||}{KBP (News)} & \multicolumn{3}{|c||}{Genomics} & \multicolumn{3}{|c|}{Pharmacogenomics}\\
\hline
Features & Method & Prec. & Rec. & F1 & Prec. & Rec. & F1 & Prec. & Rec. & F1\\
\hline
\multirow{3}{*}{Hand-tuned} & ITR & \textbf{51.15} & 26.72 & 35.10 & 83.76 & 41.67 & 55.65 & 68.16 & 49.32 & 57.23\\
& DP & 50.52 & \textbf{29.21} & \textbf{37.02} & \textbf{83.90} & 43.43 & 57.24 & \textbf{68.36} & 54.80 & \textbf{60.83} \\
\hline
\multirow{2}{*}{LSTM} & ITR & 37.68 & 28.81 & 32.66 & 69.07 & \textbf{50.76} & 58.52 & 32.35 & 43.84 & 37.23 \\
& DP & 47.47 & 27.88 & 35.78 & 75.48 & 48.48 & \textbf{58.99} & 37.63 & 47.95 & 42.17 \\
\hline
\end{tabular}
\caption{Precision/Recall/F1 scores using data programming (DP), as compared to distant supervision ITR approach, with both hand-tuned and LSTM-generated features.}
\label{exps_table_2}
\end{table}

We seek to experimentally validate three claims about our approach. Our first claim is that
data programming can be an effective paradigm for building high quality machine 
learning systems, which we test
across three real-world relation extraction applications.
Our second claim is that data
programming can be used successfully in conjunction with automatic feature
generation methods, such as LSTM models. 
Finally, our third claim is that data programming is an intuitive
and productive framework for domain-expert users, and we report on our 
initial user studies.

\paragraph*{Relation Mention Extraction Tasks}
In the \textit{relation mention extraction} task, our objects are
relation mention \textit{candidates} $x=(e_1,e_2)$, which are pairs of
entity mentions $e_1,e_2$ in unstructured text, and our goal is to
learn a model that classifies each candidate as either a true textual
assertion of the relation $R(e_1,e_2)$ or not.  We examine a news
application from the 2014 TAC-KBP Slot Filling
challenge\footnote{\small{\url{http://www.nist.gov/tac/2014/KBP/}}},
where we extract relations between real-world entities from articles~\cite{angeli2014stanford};
a clinical genomics application, where we extract causal relations
between genetic mutations and phenotypes from the scientific
literature\footnote{\small{\url{https://github.com/HazyResearch/dd-genomics}}}; and a pharmacogenomics application where we extract
interactions between genes, also from the scientific literature~\cite{mallory2015large}; further details are included in the Appendix.

For each application, we or our collaborators originally built a system
where a training set was programmatically generated 
by ordering the labeling functions as a sequence of if-then-return
statements, and for each candidate, taking the first label emitted by this 
script as the training label.  We refer to this as the
\textit{if-then-return (ITR)} approach, and note that it often
required significant
domain expert development time to tune (weeks or more).
For this set of experiments, we then used the same labeling function sets within
the framework of data programming.
For all experiments, we evaluated on a blind hand-labeled evaluation set.
In Table 1, we see that we achieve consistent improvements: on average
by 2.34 points in F1 score, including
what would have been a winning score on the 2014 TAC-KBP
challenge~\cite{surdeanu2014overview}.

We observed these performance gains across applications
with very different labeling function sets.
We describe the labeling function summary
statistics---\textit{coverage} is the percentage of objects
that had at least one label, \textit{overlap} is the percentage of
objects with more than one label, and \textit{conflict} is the
percentage of objects with conflicting labels---and see in Table 2 that
even in scenarios where $m$ is small, and conflict and overlap is relatively less common,
we still realize performance gains.
Additionally, on a disease mention extraction task (see Usability Study), which was written
from scratch within the data programming paradigm, allowing developers to
supply dependencies of the basic types outlined in Sec.~\ref{sec:dep}
led to a 2.3 point F1 score boost.

\begin{table}[h]
\centering
\small
\begin{tabular}{|l||c|c|c|c|c||c|c|}
\hline
  \multirow{2}{*}{Application} & \multirow{2}{*}{\# of LFs} & \multirow{2}{*}{Coverage} & \multirow{2}{*}{$|S_{\lambda\neq0}|$} & \multirow{2}{*}{Overlap} & \multirow{2}{*}{Conflict} & \multicolumn{2}{|c|}{F1 Score Improvement}\\
& & & & & & HT & LSTM\\
\hline 
KBP (News) & 40  & 29.39 & 2.03M & 1.38 & 0.15 & 1.92 & 3.12 \\
\hline
Genomics & 146 & 53.61 & 256K & 26.71 & 2.05 & 1.59 & 0.47 \\
\hline
Pharmacogenomics & 7 & 7.70 & 129K & 0.35 & 0.32 & 3.60 & 4.94 \\
\hline
Diseases & 12 & 53.32 & 418K & 31.81 & 0.98 & N/A & N/A \\
\hline
\end{tabular}
\caption{\small Labeling function (LF) summary statistics, sizes of generated training sets $S_{\lambda\neq0}$ (only counting non-zero labels), and relative F1 score improvement over baseline IRT methods for hand-tuned (HT) and LSTM-generated (LSTM) feature sets.}
\label{tab:lf-stats}
\end{table}


\paragraph*{Automatically-generated Features}
We additionally compare both hand-tuned and automatically-generated features, where the latter
are learned via an LSTM recurrent neural network (RNN)~\cite{hochreiter1997long}.
Conventional wisdom states that deep learning methods such as RNNs are prone to 
overfitting to the biases of the imperfect rules used for programmatic supervision.
In our experiments, however, we find that using data programming to denoise the labels can mitigate this issue, and we report
a 9.79 point boost to precision and a 3.12 point F1 score
improvement on the benchmark 2014 TAC-KBP (News)
task, over the baseline \textit{if-then-return} approach.  Additionally for comparison,
our approach is a 5.98 point F1 score improvement over a
state-of-the-art LSTM approach~\cite{verga2015multilingual}.

\paragraph*{Usability Study}
One of our hopes is that a user without expertise in ML will be more productive
iterating on labeling functions than on features.  To test this,
we arranged a hackathon involving a handful of bioinformatics
researchers, using our open-source information extraction framework Snorkel\footnote{\small{\url{snorkel.stanford.edu}}} (formerly DDLite). Their goal was to build a disease tagging system which is a common
and important challenge in the bioinformatics domain~\cite{dougan2012improved}.
The hackathon participants did
not have access to a labeled training set nor did they perform any feature
engineering. The entire effort was restricted to
iterative labeling function development and the setup of candidates
to be classified. In under eight hours, they had created a training set that led to a model which scored within 10
points of F1 of the supervised baseline; the gap was mainly due to
recall issue in the candidate extraction phase. This suggests data
programming may be a promising way to build high quality extractors, quickly.

\section{Conclusion and Future Work}

We introduced data programming, a new approach to generating 
large labeled training sets. We demonstrated that our approach can be used
with automatic feature generation techniques to achieve high quality
results. We also provided anecdotal
evidence that our methods may be easier for domain experts 
to use. We hope to explore the limits of our
approach on other machine learning tasks that have been held back by the lack of high-quality supervised datasets, including those in other domains such
imaging and structured prediction.

\paragraph*{Acknowledgements}
Thanks to Theodoros Rekatsinas, Manas Joglekar, Henry Ehrenberg, Jason Fries, Percy Liang, the \texttt{DeepDive} and \texttt{DDLite} users and many others for their helpful conversations. The authors acknowledge the support of:  DARPA FA8750-12-2-0335;  NSF IIS-1247701;  NSFCCF-1111943; DOE 108845; NSF CCF-1337375; DARPA FA8750-13-2-0039; NSF IIS-1353606;ONR N000141210041 and N000141310129; NIH U54EB020405; DARPA’s SIMPLEX program; Oracle; NVIDIA; Huawei; SAP Labs; Sloan Research Fellowship; Moore Foundation; American Family Insurance; Google; and Toshiba. The views and conclusions expressed in this material are those of the authors and should not be interpreted as necessarily representing the official policies or endorsements, either expressed or implied, of DARPA, AFRL, NSF, ONR, NIH, or the U.S. Government.

{\compactify
\bibliography{data_programming}{}

\begin{thebibliography}{10}

\bibitem{alfonseca2012pattern}
E.~Alfonseca, K.~Filippova, J.-Y. Delort, and G.~Garrido.
\newblock Pattern learning for relation extraction with a hierarchical topic
  model.
\newblock In {\em Proceedings of the ACL}.

\bibitem{angeli2014stanford}
G.~Angeli, S.~Gupta, M.~Jose, C.~D. Manning, C.~R{\'e}, J.~Tibshirani, J.~Y.
  Wu, S.~Wu, and C.~Zhang.
\newblock Stanford’s 2014 slot filling systems.
\newblock {\em TAC KBP}, 695, 2014.

\bibitem{balsubramani2015scalable}
A.~Balsubramani and Y.~Freund.
\newblock Scalable semi-supervised aggregation of classifiers.
\newblock In {\em Advances in Neural Information Processing Systems}, pages
  1351--1359, 2015.

\bibitem{NIPS2014_5253}
D.~Berend and A.~Kontorovich.
\newblock Consistency of weighted majority votes.
\newblock In {\em NIPS 2014}.

\bibitem{blum1998combining}
A.~Blum and T.~Mitchell.
\newblock Combining labeled and unlabeled data with co-training.
\newblock In {\em Proceedings of the eleventh annual conference on
  Computational learning theory}, pages 92--100. ACM, 1998.

\bibitem{bootkrajang2012label}
J.~Bootkrajang and A.~Kab{\'a}n.
\newblock Label-noise robust logistic regression and its applications.
\newblock In {\em Machine Learning and Knowledge Discovery in Databases}, pages
  143--158. Springer, 2012.

\bibitem{bunescu2007learning}
R.~Bunescu and R.~Mooney.
\newblock Learning to extract relations from the web using minimal supervision.
\newblock In {\em Annual meeting-association for Computational Linguistics},
  volume~45, page 576, 2007.

\bibitem{craven1999constructing}
M.~Craven, J.~Kumlien, et~al.
\newblock Constructing biological knowledge bases by extracting information
  from text sources.
\newblock In {\em ISMB}, volume 1999, pages 77--86, 1999.

\bibitem{Dalvi:2013:ACB:2488388.2488414}
N.~Dalvi, A.~Dasgupta, R.~Kumar, and V.~Rastogi.
\newblock Aggregating crowdsourced binary ratings.
\newblock In {\em Proceedings of the 22Nd International Conference on World
  Wide Web}, WWW '13, pages 285--294, 2013.

\bibitem{dawid1979maximum}
A.~P. Dawid and A.~M. Skene.
\newblock Maximum likelihood estimation of observer error-rates using the em
  algorithm.
\newblock {\em Applied statistics}, pages 20--28, 1979.

\bibitem{dougan2012improved}
R.~I. Do{\u{g}}an and Z.~Lu.
\newblock An improved corpus of disease mentions in pubmed citations.
\newblock In {\em Proceedings of the 2012 workshop on biomedical natural
  language processing}.

\bibitem{ehrenberg2016data}
H.~R. Ehrenberg, J.~Shin, A.~J. Ratner, J.~A. Fries, and C.~R{\'e}.
\newblock Data programming with ddlite: putting humans in a different part of
  the loop.
\newblock In {\em HILDA@ SIGMOD}, page~13, 2016.

\bibitem{gao2011harnessing}
H.~Gao, G.~Barbier, R.~Goolsby, and D.~Zeng.
\newblock Harnessing the crowdsourcing power of social media for disaster
  relief.
\newblock Technical report, DTIC Document, 2011.

\bibitem{hochreiter1997long}
S.~Hochreiter and J.~Schmidhuber.
\newblock Long short-term memory.
\newblock {\em Neural computation}, 9(8):1735--1780, 1997.

\bibitem{hoffmann2011knowledge}
R.~Hoffmann, C.~Zhang, X.~Ling, L.~Zettlemoyer, and D.~S. Weld.
\newblock Knowledge-based weak supervision for information extraction of
  overlapping relations.
\newblock In {\em Proceedings of the ACL}.

\bibitem{joglekar2015comprehensive}
M.~Joglekar, H.~Garcia-Molina, and A.~Parameswaran.
\newblock Comprehensive and reliable crowd assessment algorithms.
\newblock In {\em Data Engineering (ICDE), 2015 IEEE 31st International
  Conference on}.

\bibitem{karger2011iterative}
D.~R. Karger, S.~Oh, and D.~Shah.
\newblock Iterative learning for reliable crowdsourcing systems.
\newblock In {\em Advances in neural information processing systems}, pages
  1953--1961, 2011.

\bibitem{krishna2016visual}
R.~Krishna, Y.~Zhu, O.~Groth, J.~Johnson, K.~Hata, J.~Kravitz, S.~Chen,
  Y.~Kalantidis, L.-J. Li, D.~A. Shamma, et~al.
\newblock Visual genome: Connecting language and vision using crowdsourced
  dense image annotations.
\newblock {\em arXiv preprint arXiv:1602.07332}, 2016.

\bibitem{krogel2004multi}
M.-A. Krogel and T.~Scheffer.
\newblock Multi-relational learning, text mining, and semi-supervised learning
  for functional genomics.
\newblock {\em Machine Learning}, 57(1-2):61--81, 2004.

\bibitem{LUGOSI199279}
G.~Lugosi.
\newblock Learning with an unreliable teacher.
\newblock {\em Pattern Recognition}, 25(1):79 -- 87, 1992.

\bibitem{mallory2015large}
E.~K. Mallory, C.~Zhang, C.~R{\'e}, and R.~B. Altman.
\newblock Large-scale extraction of gene interactions from full-text literature
  using deepdive.
\newblock {\em Bioinformatics}, 2015.

\bibitem{Mintz2009}
M.~Mintz, S.~Bills, R.~Snow, and D.~Jurafsky.
\newblock Distant supervision for relation extraction without labeled data.
\newblock In {\em Proceedings of the Joint Conference of the 47th Annual
  Meeting of the ACL}, 2009.

\bibitem{NIPS2013_5073}
N.~Natarajan, I.~S. Dhillon, P.~K. Ravikumar, and A.~Tewari.
\newblock Learning with noisy labels.
\newblock In {\em Advances in Neural Information Processing Systems 26}.

\bibitem{parisi2014ranking}
F.~Parisi, F.~Strino, B.~Nadler, and Y.~Kluger.
\newblock Ranking and combining multiple predictors without labeled data.
\newblock {\em Proceedings of the National Academy of Sciences},
  111(4):1253--1258, 2014.

\bibitem{riedel2010modeling}
S.~Riedel, L.~Yao, and A.~McCallum.
\newblock Modeling relations and their mentions without labeled text.
\newblock In {\em Machine Learning and Knowledge Discovery in Databases}, pages
  148--163. Springer, 2010.

\bibitem{roth2013feature}
B.~Roth and D.~Klakow.
\newblock Feature-based models for improving the quality of noisy training data
  for relation extraction.
\newblock In {\em Proceedings of the 22nd ACM Conference on Knowledge
  management}.

\bibitem{roth2013combining}
B.~Roth and D.~Klakow.
\newblock Combining generative and discriminative model scores for distant
  supervision.
\newblock In {\em EMNLP}, pages 24--29, 2013.

\bibitem{schapire2012boosting}
R.~E. Schapire and Y.~Freund.
\newblock {\em Boosting: Foundations and algorithms}.
\newblock MIT press, 2012.

\bibitem{shin2015incremental}
J.~Shin, S.~Wu, F.~Wang, C.~De~Sa, C.~Zhang, and C.~R{\'e}.
\newblock Incremental knowledge base construction using deepdive.
\newblock {\em Proceedings of the VLDB Endowment}, 8(11):1310--1321, 2015.

\bibitem{surdeanu2014overview}
M.~Surdeanu and H.~Ji.
\newblock Overview of the english slot filling track at the tac2014 knowledge
  base population evaluation.
\newblock In {\em Proc. Text Analysis Conference (TAC2014)}, 2014.

\bibitem{takamatsu2012reducing}
S.~Takamatsu, I.~Sato, and H.~Nakagawa.
\newblock Reducing wrong labels in distant supervision for relation extraction.
\newblock In {\em Proceedings of the ACL}.

\bibitem{verga2015multilingual}
P.~Verga, D.~Belanger, E.~Strubell, B.~Roth, and A.~McCallum.
\newblock Multilingual relation extraction using compositional universal
  schema.
\newblock {\em arXiv preprint arXiv:1511.06396}, 2015.

\bibitem{NIPS2014_5431}
Y.~Zhang, X.~Chen, D.~Zhou, and M.~I. Jordan.
\newblock Spectral methods meet em: A provably optimal algorithm for
  crowdsourcing.
\newblock In {\em Advances in Neural Information Processing Systems 27}, pages
  1260--1268. 2014.

\end{thebibliography}
\bibliographystyle{abbrv}
}
\newpage

\appendix

\section{General Theoretical Results}

In this section, we will state the full form of the theoretical results we
alluded to in the body of the paper.  First, we restate, in long form, 
our setup and assumptions.

We assume that, for some function
$h: \{-1,0,1\}^m \times \{-1,1\} \mapsto \{-1,0,1\}^M$ of
\emph{sufficient statistics}, we are concerned with learning distributions,
over the set $\Omega = \{-1,0,1\}^m \times \{-1,1\}$,
of the form
\begin{equation}
  \label{eqnExpFamilyGeneral}
  \pi_{\theta}(\Lambda, Y)
  =
  \frac{1}{Z_{\theta}}
  \exp(\theta^T h(\Lambda, Y)),
\end{equation}
where $\theta \in \R^M$ is a parameter, and $Z_{\theta}$ is the partition
function that makes this a distribution.  We assume that we are given,
i.e. can derive from the data programming specification, some set $\Theta$
of \emph{feasible parameters}.  This set must have the following two 
properties.

First, for any $\theta \in \Theta$, learning the parameter $\theta$ from
(full) samples from $\pi_{\theta}$ is possible, at least in some sense.
More specifically, there exists an unbiased estimator $\hat \theta$
that is a function of some number
$D$ samples from $\pi_{\theta}$ (and is unbiased
for all $\theta \in \Theta$) such that, for all $\theta \in \Theta$
and for some $c > 0$,
\begin{equation}
  \label{eqnDataProgAssumption1}
  \Cov{\hat \theta} \preceq \frac{I}{2 c D}.
\end{equation}

Second, for any $\theta_1, \theta_2 \in \Theta$,
\begin{equation}
  \label{eqnDataProgAssumption2}
  \Exv[(\lambda_2, y_2) \sim \pi_{\theta_2}]{
    \Varc[(\lambda_1,y_1) \sim \pi_{\theta_1}]{y_1}{\lambda_1 = \lambda_2}
  }
  \le
  \frac{c}{M}.
\end{equation}
That is, we'll always be reasonably certain in our guess for the value of
$y$, even if we are totally wrong about the true parameter $\theta$.

On the other hand, we are also concerned with a distribution $\pi^*$ which
ranges over the set $\mathcal{X} \times \{-1,1\}$, and represents the
distribution of training and test examples we are using to learn.  These
objects are associated with a labeling function
$\lambda: \mathcal{X} \mapsto \{-1,0,1\}^m$
and a feature function $f: \mathcal{X} \mapsto \R^n$.
We make
three assumptions about this distribution.  First, we assume that, given
$(x, y) \sim \pi^*$,
the class label $y$ is independent of the features $f(x)$ given the labels
$\lambda(x)$.  That is,
\begin{equation}
  \label{eqnDataProgAssumption3}
  (x, y) \sim \pi^*
  \Rightarrow
  y \perp f(x)\ |\ \lambda(x).
\end{equation}
Second, we assume that we can describe the relationship between $\lambda(x)$
and $y$ in terms of our family in (\ref{eqnExpFamilyGeneral})
above.  That is, for some parameter $\theta^* \in \Theta$,
\begin{equation}
  \label{eqnDataProgAssumption4}
  \Prob[(x,y) \sim \pi^*]{\lambda(x) = \Lambda, \, y = Y}
  =
  \pi_{\theta^*}(\Lambda, Y).
\end{equation}
Third, we assume that the features themselves are bounded; for all
$x \in \mathcal{X}$,
\begin{equation}
  \label{eqnDataProgAssumption5}
  \norm{f(x)} \le 1.
\end{equation}

Our goal is twofold.  First, we want to recover some estimate $\hat \theta$
of the true parameter $\theta^*$.  Second, we want to produce a parameter
$\hat w$ that minimizes the regularized logistic loss
\[
  l(w)
  =
  \Exv[(x,y) \sim \pi^*]{
    \log(1 + \exp(-w^T f(x) y))
  }
  +
  \rho \norm{w}^2.
\]
We actually accomplish this by minimizing a noise-aware loss function,
given our recovered parameter $\hat \theta$,
\[
  l_{\hat \theta}(w)
  =
  \Exv[(\bar x, \bar y) \sim \pi^*]{
    \Exvc[(\Lambda, Y) \sim \pi_{\hat \theta}]{
      \log(1 + \exp(-w^T f(\bar x) Y))
    }{
      \Lambda = \lambda(\bar x)
    }
  }
  +
  \rho \norm{w}^2.
\]
In fact we can't even minimize this; rather, we will be minimizing the
empirical noise-aware loss function, which is only this in expectation.
Since the analysis of logistic regression is not itself interesting, we
assume that we are able to run some algorithm that produces an estimate
$\hat w$ which satisfies, for some $\chi > 0$,
\begin{equation}
  \label{eqnDataProgAssumption6}
  \Exvc{
    l_{\hat \theta}(\hat w)
    -
    \min_w l_{\hat \theta}(w)
  }{\hat \theta}
  \le
  \chi.
\end{equation}
The algorithm chosen can be anything, but in practice, we use stochastic
gradient descent.

We learn $\hat \theta$ and $\hat w$ by running the following algorithm.

\begin{algorithm}[h]
  \caption{Data Programming}
  \label{algDataProgramming}
  \begin{algorithmic}
    \REQUIRE Step size $\eta$, dataset
      $S \subset \mathcal{X}$, and initial parameter $\theta_0 \in \Theta$.
    \STATE $\theta \rightarrow \theta_0$
    \FOR{\textbf{all} $x \in S$}
      \STATE Independently sample $(\Lambda, Y)$ from $\pi_{\theta}$, and
        $(\bar \Lambda, \bar Y)$ from $\pi_{\theta}$ conditionally
        given $\Lambda = \lambda(x)$.
      \STATE $\theta \leftarrow \theta + \eta (
        h(\Lambda, Y)
        -
        h(\bar \Lambda, \bar Y)
      )$.
      \STATE $\theta = P_{\Theta}(\theta)$ \qquad
      \(\triangleright\) \textit{Here, $P_{\Theta}$ denotes orthogonal projection onto
      $\Theta$.}
    \ENDFOR
    \STATE Compute $\hat w$ using the algorithm described in
      (\ref{eqnDataProgAssumption5})
    \RETURN $(\theta, \hat w)$.
  \end{algorithmic}
\end{algorithm}

Under these assumptions, we are able to prove the following theorem
about the behavior of Algorithm \ref{algDataProgramming}.

\begin{restatable}{rtheorem}{thmDataProgramming}
  \label{thmDataProgramming}
  Suppose that we run Algorithm \ref{algDataProgramming} on a data programming
  specification that satisfies conditions
  (\ref{eqnDataProgAssumption1}),
  (\ref{eqnDataProgAssumption2}),
  (\ref{eqnDataProgAssumption3}),
  (\ref{eqnDataProgAssumption4}),
  (\ref{eqnDataProgAssumption5}), and
  (\ref{eqnDataProgAssumption6}).
  Suppose further that, for some parameter $\epsilon > 0$, we use
  step size
  \[
    \eta = \frac{c \epsilon^2}{4}
  \]
  and our dataset is of a size that satisfies
  \[
    \Abs{S}
    =
    \frac{2}{c^2 \epsilon^2}
    \log\left(
      \frac{2 \norm{\theta_0 - \theta^*}^2}{\epsilon}
    \right).
  \]
  Then, we can bound the expected parameter error with
  \[
    \Exv{\norm{\hat \theta - \theta^*}^2}
    \le
    \epsilon^2 M
  \]
  and the expected risk with
  \[
    \Exv{
      l(\hat w)
      -
      \min_w
      l(w)
    }
    \le
    \chi
    +
    \frac{
      c \epsilon
    }{
      2 \rho
    }.
  \]
\end{restatable}

This theorem's conclusions and assumptions can readily be seen to be identical
to those of Theorem \ref{thmDependentDataProgramming} in the main body of 
the paper, except that they apply to the slightly more general case of 
arbitrary $h$, rather than $h$ of the explicit form described in the body.
Therefore, in order to prove Theorem \ref{thmDependentDataProgramming},
it suffices to prove Theorem \ref{thmDataProgramming}, which we will do 
in Section \ref{sAppProofThmDP}.

\section{Theoretical Results for Independent Model}

For the independent model, we can obtain a more specific version of
Theorem \ref{thmDataProgramming}.  In the independent model, the 
variables are, as before, $\Lambda \in \{-1,0,1\}^m$ and $Y \in \{-1,1\}$.
The sufficient statistics are $\Lambda_i Y$ and $\Lambda_i^2$.

To produce results that make intuitive sense, we also define the alternate
parameterization
\[
  \Probc[\pi]{\Lambda_i}{Y}
  =
  \begin{cases}
    \beta_i \frac{1 + \gamma_i}{2} & \Lambda_i = Y \\
    (1 - \beta_i) & \Lambda = 0 \\
    \beta_i \frac{1 - \gamma_i}{2} & \Lambda_i = -Y
  \end{cases}.
\]
In comparison to the parameters used in the body of the paper, 
we have
\[
  \alpha_i = \frac{1 + \gamma_i}{2}.
\]

Now, we are concerned with models that are feasible.  For a model to be
feasible (i.e. for $\theta \in \Theta$),
we require that it satisfy, for some constants $\gamma_{\min} > 0$,
$\gamma_{\max} > 0$, and $\beta_{\min}$,
\[
  \gamma_{\min} \le \gamma_i \le \gamma_{\max}
  \qquad
  \beta_{\min} \le \beta_i \le \frac{1}{2}.
\]
For $0 \le \beta \le 1$ and $-1 \le \gamma \le 1$.

For this model, we can prove the following corollary to Theorem
\ref{thmDataProgramming}

\begin{restatable}{rcorollary}{corDataProgrammingIndep}
  \label{corDataProgrammingIndep}
  Suppose that we run Algorithm \ref{algDataProgramming} on an
  independent data programming
  specification that satisfies conditions
  (\ref{eqnDataProgAssumption3}),
  (\ref{eqnDataProgAssumption4}),
  (\ref{eqnDataProgAssumption5}), and
  (\ref{eqnDataProgAssumption6}).
  Furthermore, assume that the number of labeling functions we use
  satisfies
  \[
    m
    \ge
    \frac{
      9.34 \artanh(\gamma_{\max})
    }{
      (\gamma \beta)_{\min} \gamma_{\min}^2
    }
    \log\left(
      \frac{24 m}{\beta_{\min}}
    \right).
  \]
  Suppose further that, for some parameter $\epsilon > 0$, we use
  step size
  \[
    \eta = \frac{\beta_{\min} \epsilon^2}{16}
  \]
  and our dataset is of a size that satisfies
  \[
    \Abs{S}
    =
    \frac{32}{\beta_{\min}^2 \epsilon^2}
    \log\left(
      \frac{2 \norm{\theta_0 - \theta^*}^2}{\epsilon}
    \right).
  \]
  Then, we can bound the expected parameter error with
  \[
    \Exv{\norm{\hat \theta - \theta^*}^2}
    \le
    \epsilon^2 M
  \]
  and the expected risk with
  \[
    \Exv{
      l(\hat w)
      -
      \min_w
      l(w)
    }
    \le
    \chi
    +
    \frac{
      \beta_{\min} \epsilon
    }{
      8 \rho
    }.
  \]
\end{restatable}

We can see that if, as stated in the body of the paper, $\beta_i \ge 0.3$
and $0.8 \le \alpha_i \le 0.9$
(which is equivalent to $0.6 \le \gamma_i \le 0.8$), then
\[
  2000
  \ge
  1896.13
  =
  \frac{
    9.34 \artanh(0.8)
  }{
    0.3 \cdot 0.6^3
  }
  \log\left(
    \frac{24 \cdot 2000}{0.3}
  \right).
\]
This means that, as stated in the paper, $m = 2000$ is sufficient for this
corollary to hold with 
\[
  \Abs{S}
  =
  \frac{32}{0.3^2 \cdot \epsilon^2}
  \log\left(
    \frac{2 m (\artanh(0.8) - \artanh(0.6))^2}{\epsilon}
  \right)
  =
  \frac{356}{\epsilon^2}
  \log\left(
    \frac{m}{3 \epsilon}
  \right).
\]
Thus, proving Corollary \ref{corDataProgrammingIndep} is sufficient to prove
Theorem \ref{stmtIndepDataProgramming} from the body of the paper.  We 
will prove Corollary \ref{corDataProgrammingIndep} in Section
\ref{secAppIndepModelProofs}

\section{Proof of Theorem \ref{thmDataProgramming}}
\label{sAppProofThmDP}

First, we state some lemmas that will be useful in the proof to come.

\begin{restatable*}{rlemma}{lemmaMLEDerivatives}
  \label{lemmaMLEDerivatives}
  Given a family of maximum-entropy distributions
  \[
    \pi_{\theta}(x)
    =
    \frac{1}{Z_{\theta}}
    \exp(\theta^T h(x)),
  \]
  for some function of sufficient statistics $h: \Omega \mapsto \R^M$,
  if we let $J: \R^M \mapsto \R$ be the maximum log-likelihood objective
  for some event $A \subseteq \Omega$,
  \[
    J(\theta)
    =
    \log \Prob[x \sim \pi_\theta]{x \in A},
  \]
  then its gradient is
  \[
    \nabla J(\theta)
    =
    \Exvc[x \sim \pi_{\theta}]{h(x)}{x \in A}
    -
    \Exv[x \sim \pi_{\theta}]{h(x)}
  \]
  and its Hessian is
  \[
    \nabla^2 J(\theta)
    =
    \Covc[x \sim \pi_{\theta}]{h(x)}{x \in A}
    -
    \Cov[x \sim \pi_{\theta}]{h(x)}.
  \]
\end{restatable*}

\begin{restatable*}{rlemma}{lemmaHessianBound}
  \label{lemmaHessianBound}
  Suppose that we are looking at a distribution from a data programming
  label model.  That is, our maximum-entropy distribution can now be written in
  terms of two variables, the labeling function values
  $\lambda \in \{-1,0,1\}$ and the class $y \in \{-1,1\}$, as
  \[
    \pi_{\theta}(\lambda, y)
    =
    \frac{1}{Z_{\theta}}
    \exp(\theta^T h(\lambda, y)),
  \]
  where we assume without loss of generality that for some $M$,
  $h(\lambda, y) \in \R^M$ and $\norm{h(\lambda, y)}_{\infty} \le 1$.
  If we let $J: \R^M \mapsto \R$ be the maximum expected log-likelihood
  objective, under another distribution $\pi^*$, for the event associated
  with the observed labeling function values $\lambda$,
  \[
    J(\theta)
    =
    \Exv[(\lambda^*, y^*) \sim \pi^*]{
      \log \Prob[(\lambda, y) \sim \pi_\theta]{\lambda = \lambda^*}
    },
  \]
  then its Hessian can be bounded with
  \[
    \nabla^2 J(\theta)
    \preceq
    M I
    \Exv[(\lambda^*, y^*) \sim \pi^*]{
      \Varc[(\lambda,y) \sim \pi_{\theta}]{y}{\lambda = \lambda^*}
    }
    -
    \mathcal{I}(\theta),
  \]
  where $\mathcal{I}(\theta)$ is the Fisher information.
\end{restatable*}

\begin{restatable*}{rlemma}{lemmaSuffCondConvex}
  \label{lemmaSuffCondConvex}
  Suppose that we are looking at a data programming distribution, as
  described in the text of Lemma \ref{lemmaHessianBound}.
  Suppose further that we are concerned with some feasible set of parameters
  $\Theta \subset \R^M$, such that the any model with parameters in this
  space satisfies the following two conditions.

  First, for any $\theta \in \Theta$, learning the parameter $\theta$ from
  (full) samples from $\pi_{\theta}$ is possible, at least in some sense.
  More specifically, there exists an unbiased estimator $\hat \theta$
  that is a function of some number
  $D$ samples from $\pi_{\theta}$ (and is unbiased
  for all $\theta \in \Theta$) such that, for all $\theta \in \Theta$
  and for some $c > 0$,
  \[
    \Cov{\hat \theta} \preceq \frac{I}{2 c D}.
  \]

  Second, for any $\theta, \theta^* \in \Theta$,
  \[
    \Exv[(\lambda^*, y^*) \sim \pi^*]{
      \Varc[(\lambda,y) \sim \pi_{\theta}]{y}{\lambda = \lambda^*}
    }
    \le
    \frac{c}{M}.
  \]
  That is, we'll always be reasonably certain in our guess for the value of
  $y$, even if we are totally wrong about the true parameter $\theta^*$.

  Under these conditions, the function $J$ is strongly concave on $\Theta$
  with parameter of strong convexity $c$.
\end{restatable*}

\begin{restatable*}{rlemma}{lemmaSGDConvex}
  \label{lemmaSGDConvex}
  Suppose that we are looking at a data programming maximum likelihood
  estimation problem, as
  described in the text of Lemma \ref{lemmaHessianBound}.
  Suppose further that the objective function $J$ is strongly concave
  with parameter $c > 0$.

  If we run stochastic gradient descent on objective $J$, using unbiased
  samples from a true distribution $\pi_{\theta^*}$, where
  $\theta^* \in \Theta$, then if we use step size
  \[
    \eta
    =
    \frac{c \epsilon^2}{4}
  \]
  and run (using a fresh sample at each iteration) for $T$ steps, where
  \[
    T
    =
    \frac{2}{c^2 \epsilon^2}
    \log\left(
      \frac{2 \norm{\theta_0 - \theta^*}^2}{\epsilon}
    \right)
  \]
  then we can bound the expected parameter estimation error with
  \[
    \Exv{\norm{\hat \theta - \theta^*}^2}
    \le
    \epsilon^2 M.
  \]
\end{restatable*}

\begin{restatable*}{rlemma}{lemmaExpectedLoss}
  \label{lemmaExpectedLoss}
  Assume in our model that, without loss of generality, $\norm{f(x)} \le 1$ for
  all $x$, and that in our true model $\pi^*$, the class $y$ is independent
  of the features $f(x)$ given the labels $\lambda(x)$.

  Suppose that we now want to solve the expected loss minimization problem
  wherein we minimize the objective
  \[
    l(w)
    =
    \Exv[(x,y) \sim \pi^*]{
      \log(1 + \exp(-w^T f(x) y))
    }
    +
    \rho \norm{w}^2.
  \]
  We actually accomplish this by minimizing our noise-aware loss function,
  given our chosen parameter $\hat \theta$,
  \[
    l_{\hat \theta}(w)
    =
    \Exv[(\bar x, \bar y) \sim \pi^*]{
      \Exvc[(\Lambda, Y) \sim \pi_{\hat \theta}]{
        \log(1 + \exp(-w^T f(\bar x) Y))
      }{
        \Lambda = \lambda(\bar x)
      }
    }
    +
    \rho \norm{w}^2.
  \]
  In fact we can't even minimize this; rather, we will be minimizing the
  empirical noise-aware loss function, which is only this in expectation.
  Suppose that doing so produces an estimate $\hat w$ which satisfies,
  for some $\chi > 0$,
  \[
    \Exvc{
      l_{\hat \theta}(\hat w)
      -
      \min_w l_{\hat \theta}(w)
    }{\hat \theta}
    \le
    \chi.
  \]
  (Here, the expectation is taken with respect to only the random variable
  $\hat w$.)  Then, we can bound the expected risk with
  \[
    \Exv{
      l(\hat w)
      -
      \min_w
      l(w)
    }
    \le
    \chi
    +
    \frac{
      c \epsilon
    }{
      2 \rho
    }.
  \]
\end{restatable*}

Now, we restate and prove our main theorem.

\thmDataProgramming*

\begin{proof}
  The bounds on the expected parameter estimation error follow directly from
  Lemma \ref{lemmaSGDConvex}, and the remainder of the theorem follows
  directly from Lemma \ref{lemmaExpectedLoss}.
\end{proof}

\section{Proofs of Lemmas}

\lemmaMLEDerivatives

\begin{proof}
  For the gradient,
  \begin{align*}
    \nabla J(\theta)
    &=
    \nabla \log \Prob[\pi_\theta]{A} \\
    &=
    \nabla \log\left(
      \frac{
        \sum_{x \in A} \exp(\theta^T h(x))
      }{
        \sum_{x \in \Omega} \exp(\theta^T h(x))
      }
    \right) \\
    &=
    \nabla \log\left(
      \sum_{x \in A} \exp(\theta^T h(x))
    \right)
    -
    \nabla \log\left(
      \sum_{x \in \Omega} \exp(\theta^T h(x))
    \right) \\
    &=
    \frac{
      \sum_{x \in A} h(x) \exp(\theta^T h(x))
    }{
      \sum_{x \in A} \exp(\theta^T h(x))
    }
    -
    \frac{
      \sum_{x \in \Omega} h(x) \exp(\theta^T h(x))
    }{
      \sum_{x \in \Omega} \exp(\theta^T h(x))
    } \\
    &=
    \Exvc[x \sim \pi_{\theta}]{h(x)}{x \in A}
    -
    \Exv[x \sim \pi_{\theta}]{h(x)}.
  \end{align*}
  And for the Hessian,
  \begin{align*}
    \nabla^2 J(\theta)
    &=
    \nabla
    \frac{
      \sum_{x \in A} h(x) \exp(\theta^T h(x))
    }{
      \sum_{x \in A} \exp(\theta^T h(x))
    }
    -
    \nabla
    \frac{
      \sum_{x \in \Omega} h(x) \exp(\theta^T h(x))
    }{
      \sum_{x \in \Omega} \exp(\theta^T h(x))
    }
    \\ &=
    \frac{
      \sum_{x \in A} h(x) h(x)^T \exp(\theta^T h(x))
    }{
      \sum_{x \in A} \exp(\theta^T h(x))
    }
    -
    \frac{
      \left( \sum_{x \in A} h(x) \exp(\theta^T h(x)) \right)
      \left( \sum_{x \in A} h(x) \exp(\theta^T h(x)) \right)^T
    }{
      \left( \sum_{x \in A} \exp(\theta^T h(x)) \right)^2
    }
    \\ &\hspace{2em}-
    \left(
      \frac{
        \sum_{x \in \Omega} h(x) h(x)^T \exp(\theta^T h(x))
      }{
        \sum_{x \in \Omega} \exp(\theta^T h(x))
      }
      -
      \frac{
        \left( \sum_{x \in \Omega} h(x) \exp(\theta^T h(x)) \right)
        \left( \sum_{x \in \Omega} h(x) \exp(\theta^T h(x)) \right)^T
      }{
        \left( \sum_{x \in \Omega} \exp(\theta^T h(x)) \right)^2
      }
    \right)
    \\ &=
    \Exvc[x \sim \pi_{\theta}]{h(x) h(x)^T}{x \in A}
    -
    \Exvc[x \sim \pi_{\theta}]{h(x)}{x \in A}
    \Exvc[x \sim \pi_{\theta}]{h(x)}{x \in A}^T
    \\&\hspace{2em}-
    \left(
      \Exv[x \sim \pi_{\theta}]{h(x) h(x)^T}
      -
      \Exv[x \sim \pi_{\theta}]{h(x)}
      \Exv[x \sim \pi_{\theta}]{h(x)}^T
    \right)
    \\ &=
    \Covc[x \sim \pi_{\theta}]{h(x)}{x \in A}
    -
    \Cov[x \sim \pi_{\theta}]{h(x)}.
  \end{align*}
\end{proof}

\lemmaHessianBound

\begin{proof}
  From the result of Lemma \ref{lemmaMLEDerivatives}, we have that
  \begin{equation}
    \label{eqnLemmaHessianBound1}
    \nabla^2 J(\theta)
    =
    \Exv[(\lambda^*, y^*) \sim \pi^*]{
      \Covc[(\lambda, y) \sim \pi_{\theta}]{h(\lambda, y)}{\lambda = \lambda^*}
    }
    -
    \Cov[(\lambda, y) \sim \pi_{\theta}]{h(\lambda, y)}.
  \end{equation}
  We start byu defining $h_0(\lambda)$ and $h_1(\lambda)$ such that
  \[
    h(\lambda, y)
    =
    h(\lambda, 1) \frac{1 + y}{2}
    +
    h(\lambda, -1) \frac{1 - y}{2}
    =
    \frac{h(\lambda, 1) + h(\lambda, -1)}{2}
    +
    y \frac{h(\lambda, 1) - h(\lambda, -1)}{2}
    =
    h_0(\lambda) + y h_1(\lambda).
  \]
  This allows us to reduce (\ref{eqnLemmaHessianBound1}) to
  \[
    \nabla^2 J(\theta)
    =
    \Exv[(\lambda^*, y^*) \sim \pi^*]{
      h_1(\lambda^*) h_1(\lambda^*)^T
      \Varc[(\lambda,y) \sim \pi_{\theta}]{y}{\lambda = \lambda^*}
    }
    -
    \Cov[(\lambda,y) \sim \pi_{\theta}]{h(\lambda,y)}.
  \]
  On the other hand, the Fisher information of this model at $\theta$ is
  \begin{align*}
    \mathcal{I}(\theta)
    &=
    \Exv{
      \left(
        \nabla_{\theta} \log \pi_{\theta}(x)
      \right)^2
    }
    \\ &=
    \Exv{
      \left(
        \nabla_{\theta} \log\left(
          \frac{
            \exp(\theta^T h(x))
          }{
            \sum_{z \in \Omega}
            \exp(\theta^T h(z))
          }
        \right)
      \right)^2
    }
    \\ &=
    \Exv{
      \left(
        \nabla_{\theta} \log\left(
          \exp(\theta^T h(x))
        \right)
        -
        \nabla_{\theta} \log\left(
          \sum_{z \in \Omega}
          \exp(\theta^T h(z))
        \right)
      \right)^2
    }
    \\ &=
    \Exv{
      \left(
        h(x)
        -
        \frac{
          \sum_{z \in \Omega}
          h(z) \exp(\theta^T h(z))
        }{
          \sum_{z \in \Omega}
          \exp(\theta^T h(z))
        }
      \right)^2
    }
    \\ &=
    \Exv{
      \left(
        h(x)
        -
        \Exv{h(z)}
      \right)^2
    }
    \\ &=
    \Cov{h(x)}.
  \end{align*}
  Therefore, we can write the second derivative of $J$ as
  \[
    \nabla^2 J(\theta)
    =
    \Exv[(\lambda^*, y^*) \sim \pi^*]{
      h_1(\lambda^*) h_1(\lambda^*)^T
      \Varc[(\lambda,y) \sim \pi_{\theta}]{y}{\lambda = \lambda^*}
    }
    -
    \mathcal{I}(\theta).
  \]
  If we apply the fact that
  \[
    h_1(\lambda^*) h_1(\lambda^*)^T
    \preceq
    I \norm{h_1(\lambda^*)}^2
    \preceq
    M I \norm{h_1(\lambda^*)}_{\infty}^2
    \preceq
    M I,
  \]
  then we can reduce this to
  \[
    \nabla^2 J(\theta)
    \preceq
    M I 
    \Exv[(\lambda^*, y^*) \sim \pi^*]{
      \Varc[(\lambda,y) \sim \pi_{\theta}]{y}{\lambda = \lambda^*}
    }
    -
    \mathcal{I}(\theta).
  \]
  This is the desired result.
\end{proof}

\lemmaSuffCondConvex

\begin{proof}
  From the Cram{\'e}r-Rao bound, we know in general
  that the variance of any unbiased
  estimator is bounded by the reciprocal of the Fisher information
  \[
    \Cov{\hat \theta}
    \succeq
    \left( \mathcal{I}(\theta) \right)^{-1}.
  \]
  Since for the estimator described in the lemma statement, we have $D$
  independent samples from the distribution, it follows that the Fisher
  information of this experiment is $D$ times the Fisher information of a 
  single sample.  Combining this with the bound in the lemma statement
  on the covariance, we get
  \[
    \frac{I}{2 c D}
    \succeq
    \Cov{\hat \theta}
    \succeq
    \left( D \mathcal{I}(\theta) \right)^{-1}.
  \]
  It follows that
  \[
    \mathcal{I}(\theta) \succeq 2 c I.
  \]
  On the other hand, also from the lemma statement, we can conclude that
  \[
    M I
    \Exv[(\lambda^*, y^*) \sim \pi^*]{
      \Varc[(\lambda,y) \sim \pi_{\theta}]{y}{\lambda = \lambda^*}
    }
    \preceq
    c I.
  \]
  Therefore, for all $\theta \in \Theta$,
  \[
    \nabla^2 J(\theta)
    \preceq
    M I 
    \Exv[(\lambda^*, y^*) \sim \pi^*]{
      \Varc[(\lambda,y) \sim \pi_{\theta}]{y}{\lambda = \lambda^*}
    }
    -
    \mathcal{I}(\theta)
    \preceq
    - c I.
  \]
  This implies that $J$ is strongly concave over $\Theta$, with constant $c$,
  as desired.
\end{proof}

\lemmaSGDConvex

\begin{proof}
  First, we note that, in the proof to follow,
  we can ignore the projection onto the feasible set
  $\Theta$, since this projection always takes us closer to the optimum
  $\theta^*$.

  If we track the expected distance to the optimum $\theta^*$, then 
  at the next timestep,
  \[
    \norm{\theta_{t+1} - \theta^*}^2
    =
    \norm{\theta_t - \theta^*}^2
    +
    2 \gamma (\theta_t - \theta^*) \nabla \tilde J(\theta_t)
    +
    \gamma^2 \norm{\nabla \tilde J_t(\theta_t)}^2.
  \]
  Since we can write our stochastic samples in the form
  \[
    \nabla \tilde J_t(\theta_t)
    =
    h(\lambda_t, y_t) - h(\bar \lambda_t, \bar y_t),
  \]
  for some samples $\lambda_t$, $y_t$, $\bar \lambda_t$, and $\bar y_t$,
  we can conclude that
  \[
    \norm{\nabla \tilde J_t(\theta_t)}^2
    \le
    M \norm{\nabla \tilde J_t(\theta_t)}_{\infty}^2
    \le
    4 M.
  \]
  Therefore, taking the expected value conditioned on the filtration,
  \[
    \Exvc{\norm{\theta_{t+1} - \theta^*}^2}{\F_t}
    =
    \norm{\theta_t - \theta^*}^2
    +
    2 \gamma (\theta_t - \theta^*) \nabla J(\theta_t)
    +
    4 \gamma^2 M.
  \]
  Since $J$ is strongly concave,
  \[
    (\theta_t - \theta^*) \nabla J(\theta_t)
    \le
    -c \norm{\theta_t - \theta^*}^2;
  \]
  and so,
  \[
    \Exvc{\norm{\theta_{t+1} - \theta^*}^2}{\F_t}
    \le
    (1 - 2 \gamma c) \norm{\theta_t - \theta^*}^2
    +
    4 \gamma^2 M.
  \]
  If we take the full expectation and subtract the fixed point from both sides,
  \[
    \Exv{\norm{\theta_{t+1} - \theta^*}^2}
    -
    \frac{2 \gamma M}{c}
    \le
    (1 - 2 \gamma c) \Exv{\norm{\theta_t - \theta^*}^2}
    +
    4 \gamma^2 M
    -
    \frac{2 \gamma M}{c}
    =
    (1 - 2 \gamma c) \left(
      \Exv{\norm{\theta_t - \theta^*}^2}
      -
      \frac{2 \gamma M}{c}
    \right).
  \]
  Therefore,
  \[
    \Exv{\norm{\theta_t - \theta^*}^2}
    -
    \frac{2 \gamma M}{c}
    \le
    (1 - 2 \gamma c)^t \left(
      \norm{\theta_0 - \theta^*}^2
      -
      \frac{2 \gamma M}{c}
    \right),
  \]
  and so
  \[
    \Exv{\norm{\theta_t - \theta^*}^2}
    \le
    \exp(- 2 \gamma c t)
    \norm{\theta_0 - \theta^*}^2
    +
    \frac{2 \gamma M}{c}.
  \]
  In order to ensure that
  \[
    \Exv{\norm{\theta_t - \theta^*}^2}
    \le
    \epsilon^2,
  \]
  it therefore suffices to pick
  \[
    \gamma
    =
    \frac{c \epsilon^2}{4 M}
  \]
  and
  \[
    t
    =
    \frac{2 M}{c^2 \epsilon^2}
    \log\left(
      \frac{2 \norm{\theta_0 - \theta^*}^2}{\epsilon}
    \right).
  \]
  Substituting $\epsilon^2 \rightarrow \epsilon^2 M$ produces the desired
  result.
\end{proof}

\lemmaExpectedLoss

\begin{proof}
  (To simplify the symbols in this proof, we freely use $\theta$ when we 
  mean $\hat \theta$.)

  The loss function we want to minimize is, in expectation,
  \[
    l(w)
    =
    \Exv[(x,y) \sim \pi^*]{
      \log(1 + \exp(-w^T f(x) y))
    }
    +
    \rho \norm{w}^2.
  \]
  By the law of total expectation,
  \[
    l(w)
    =
    \Exv[(\bar x, \bar y) \sim \pi^*]{
      \Exvc[(x, y) \sim \pi^*]{
        \log(1 + \exp(-w^T f(\bar x) y))
      }{
        x = \bar x
      }
    }
    +
    \rho \norm{w}^2,
  \]
  and by our conditional independence assumption,
  \[
    l(w)
    =
    \Exv[(\bar x, \bar y) \sim \pi^*]{
      \Exvc[(x, y) \sim \pi^*]{
        \log(1 + \exp(-w^T f(\bar x) y))
      }{
        \lambda(x) = \lambda(\bar x)
      }
    }
    +
    \rho \norm{w}^2.
  \]
  Since we know from our assumptions that, for the optimum parameter
  $\theta^*$,
  \[
    \Prob[(x, y) \sim \pi^*]{\lambda(x) = \Lambda, y = Y}
    =
    \Prob[(\lambda, y) \sim \pi_{\theta^*}]{
      \lambda = \Lambda, y = Y
    },
  \]
  we can rewrite this as
  \[
    l(w)
    =
    \Exv[(\bar x, \bar y) \sim \pi^*]{
      \Exvc[(\Lambda, Y) \sim \pi_{\theta^*}]{
        \log(1 + \exp(-w^T f(\bar x) Y))
      }{
        \Lambda = \lambda(\bar x)
      }
    }
    +
    \rho \norm{w}^2.
  \]
  On the other hand, if we are minimizing the model we got from the previous
  step, we will be actually minimizing
  \[
    l_{\theta}(w)
    =
    \Exv[(\bar x, \bar y) \sim \pi^*]{
      \Exvc[(\Lambda, Y) \sim \pi_\theta]{
        \log(1 + \exp(-w^T f(\bar x) Y))
      }{
        \Lambda = \lambda(\bar x)
      }
    }
    +
    \rho \norm{w}^2.
  \]
  We can reduce this further by noticing that
  \begin{align*}
    &\Exvc[(\Lambda, Y) \sim \pi_\theta]{
      \log(1 + \exp(-w^T f(\bar x) Y))
    }{
      \Lambda = \lambda(\bar x)
    } \\
    &\hspace{2em}=
    \Exvc[(\Lambda, Y) \sim \pi_\theta]{
      \log(1 + \exp(-w^T f(\bar x)))
      \frac{1 + Y}{2}
      +
      \log(1 + \exp(w^T f(\bar x)))
      \frac{1 - Y}{2}
    }{
      \Lambda = \lambda(\bar x)
    } \\
    &\hspace{2em}=
    \frac{
      \log(1 + \exp(-w^T f(\bar x)))
      +
      \log(1 + \exp(w^T f(\bar x)))
    }{
      2
    } \\
    &\hspace{4em}+
    \frac{
      \log(1 + \exp(-w^T f(\bar x)))
      -
      \log(1 + \exp(w^T f(\bar x)))
    }{
      2
    }
    \Exvc[(\Lambda, Y) \sim \pi_{\theta}]{
      Y
    }{
      \Lambda = \lambda(\bar x)
    } \\
    &\hspace{2em}=
    \frac{
      \log(1 + \exp(-w^T f(\bar x)))
      +
      \log(1 + \exp(w^T f(\bar x)))
    }{
      2
    } \\
    &\hspace{4em}-
    \frac{
      w^T f(\bar x)
    }{
      2
    }
    \Exvc[(\Lambda, Y) \sim \pi_{\theta}]{
      Y
    }{
      \Lambda = \lambda(\bar x)
    }.
  \end{align*}
  It follows that the difference between the loss functions will be
  \begin{align*}
    \Abs{
      l(w)
      -
      l_{\theta}(w)
    }
    &=
    \Abs{
      \Exv[(\bar x, \bar y) \sim \pi^*]{
        \frac{
          w^T f(\bar x)
        }{
          2
        }
        \left(
          \Exvc[(\Lambda, Y) \sim \pi_{\theta}]{
            Y
          }{
            \Lambda = \lambda(\bar x)
          }
          -
          \Exvc[(\Lambda, Y) \sim \pi_{\theta^*}]{
            Y
          }{
            \Lambda = \lambda(\bar x)
          }
        \right)
      }
    }.
  \end{align*}
  Now, we can compute that
  \begin{align*}
    \nabla_{\theta}
    \Exvc[(\Lambda, Y) \sim \pi_{\theta}]{
      Y
    }{
      \Lambda = \lambda
    }
    &=
    \nabla_{\theta}
    \frac{
      \exp(\theta^T h(\lambda, 1))
      -
      \exp(\theta^T h(\lambda, -1))
    }{
      \exp(\theta^T h(\lambda, 1))
      +
      \exp(\theta^T h(\lambda, -1))
    }
    \\ &=
    \nabla_{\theta}
    \frac{
      \exp(\theta^T h_1(\lambda))
      -
      \exp(-\theta^T h_1(\lambda))
    }{
      \exp(\theta^T h_1(\lambda))
      +
      \exp(\theta^T h_1(\lambda))
    }
    \\ &=
    \nabla_{\theta}
    \tanh(\theta^T h_1(\lambda))
    \\ &=
    h_1(\lambda)
    \left(
      1
      -
      \tanh^2(\theta^T h_1(\lambda))
    \right)
    \\ &=
    h_1(\lambda)
    \Varc[(\Lambda, Y) \sim \pi_{\theta}]{
      Y
    }{
      \Lambda = \lambda
    }.
  \end{align*}
  It follows by the mean value theorem that for some $\psi$, a linear 
  combination of $\theta$ and $\theta^*$,
  \begin{align*}
    \Abs{
      l(w)
      -
      l_{\theta}(w)
    }
    &=
    \Abs{
      \Exv[(\bar x, \bar y) \sim \pi^*]{
        \frac{
          w^T f(\bar x)
        }{
          2
        }
        (\theta - \theta^*)^T 
        h_1(\lambda)
        \Varc[(\Lambda, Y) \sim \pi_{\psi}]{
          Y
        }{
          \Lambda = \lambda
        }
      }
    }.
  \end{align*}
  Since $\Theta$ is convex, clearly $\psi \in \Theta$.  From our
  assumption on the bound of the variance, we can conclude that
  \[
    \Exv[(\bar x, \bar y) \sim \pi^*]{
      \Varc[(\Lambda, Y) \sim \pi_{\psi}]{
        Y
      }{
        \Lambda = \lambda
      }
    }
    \le
    \frac{c}{M}.
  \]
  By the Cauchy-Schwarz inequality,
  \begin{align*}
    \Abs{
      l(w)
      -
      l_{\theta}(w)
    }
    &\le
    \frac{1}{2}
    \Abs{
      \Exv[(\bar x, \bar y) \sim \pi^*]{
        \norm{w} \norm{f(\bar x)}
        \norm{\theta - \theta^*} 
        \norm{h_1(\lambda)}
        \Varc[(\Lambda, Y) \sim \pi_{\psi}]{
          Y
        }{
          \Lambda = \lambda
        }
      }
    }.
  \end{align*}
  Since (by assumption) $\norm{f(x)} \le 1$ and
  $\norm{h_1(\lambda)} \le \sqrt{M}$,
  \begin{align*}
    \Abs{
      l(w)
      -
      l_{\theta}(w)
    }
    &\le
    \frac{
      \norm{w}
      \norm{\theta - \theta^*} 
      \sqrt{M}
    }{2}  
    \Abs{
      \Exv[(\bar x, \bar y) \sim \pi^*]{
        \Varc[(\Lambda, Y) \sim \pi_{\psi}]{
          Y
        }{
          \Lambda = \lambda
        }
      }
    }
    \\ &\le
    \frac{
      \norm{w}
      \norm{\theta - \theta^*} 
      \sqrt{M}
    }{2}
    \cdot
    \frac{c}{M}
    \\ &=
    \frac{
      c
      \norm{w}
      \norm{\theta - \theta^*} 
    }{
      2 \sqrt{M}
    }.
  \end{align*}
  Now, for any $w$ that could conceivably be a solution, it must be the
  case that
  \[
    \norm{w} \le \frac{1}{2 \rho},
  \]
  since otherwise the regularization term would be too large
  Therefore, for any possible solution $w$,
  \[
    \Abs{
      l(w)
      -
      l_{\theta}(w)
    }
    \le
    \frac{
      c
      \norm{\theta - \theta^*} 
    }{
      4 \rho \sqrt{M}
    }.
  \]
  Now, we apply the assumption that we are able to solve the empirical problem,
  producing an estimate $\hat w$ that satisfies
  \[
    \Exv{l_{\theta}(\hat w) - l_{\theta}(w^*_{\theta})}
    \le
    \chi,
  \]
  where $w^*_{\theta}$ is the true solution to
  \[
    w^*_{\theta}
    =
    \arg \min_w
    l_{\theta}(w).
  \]
  Therefore,
  \begin{align*}
    \Exv{
      l(\hat w)
      -
      l(w^*)
    }
    &=
    \Exv{
      l_{\theta}(\hat w)
      -
      l_{\theta}(w^*_{\theta})
      +
      l_{\theta}(w^*_{\theta})
      -
      l_{\theta}(\hat w)
      +
      l(\hat w)
      -
      l(w^*)
    }
    \\ &\le
    \chi
    +
    \Exv{
      l_{\theta}(w^*)
      -
      l_{\theta}(\hat w)
      +
      l(\hat w)
      -
      l(w^*)
    }
    \\ &\le
    \chi
    +
    \Exv{
      \Abs{
        l_{\theta}(w^*)
        -
        l(w^*)
      }
      +
      \Abs{
        l_{\theta}(\hat w)
        -
        l(\hat w)
      }
    }
    \\ &\le
    \chi
    +
    \Exv{
      \frac{
        c
        \norm{\theta - \theta^*} 
      }{
        2 \rho \sqrt{M}
      }
    }
    \\ &=
    \chi
    +
    \frac{
      c
    }{
      2 \rho \sqrt{M}
    }
    \Exv{\norm{\theta - \theta^*}}
    \\ &\le
    \chi
    +
    \frac{
      c
    }{
      2 \rho \sqrt{M}
    }
    \sqrt{\Exv{\norm{\theta - \theta^*}^2}}.
  \end{align*}
  We can now bound this using the result of Lemma \ref{lemmaSGDConvex},
  which results in
  \begin{align*}
    \Exv{
      l(\hat w)
      -
      l(w^*)
    }
    &\le
    \chi
    +
    \frac{
      c
    }{
      2 \rho \sqrt{M}
    }
    \sqrt{M \epsilon^2}
    \\ &=
    \chi
    +
    \frac{
      c \epsilon
    }{
      2 \rho
    }.
  \end{align*}
  This is the desired result.
\end{proof}

\section{Proofs of Results for the Independent Model}
\label{secAppIndepModelProofs}

To restate, in the independent model, the 
variables are, as before, $\Lambda \in \{-1,0,1\}^m$ and $Y \in \{-1,1\}$.
The sufficient statistics are $\Lambda_i Y$ and $\Lambda_i^2$.
That is, for expanded parameter $\theta = (\psi, \phi)$,
\[
  \pi_{\theta}(\Lambda, Y)
  =
  \frac{1}{Z}
  \exp(\psi^T \Lambda Y + \phi^T \Lambda^2).
\]
This can be combined with the simple assumption that $\Prob{Y} = \frac{1}{2}$
to complete a whole distribution.  Using this, we can prove the following
simple result about the moments of the sufficient statistics.
\begin{restatable}{rlemma}{lemmaSuffStatsMoments}
  \label{lemmaSuffStatsMoments}
  The expected values and covariances of the sufficient statistics are, for all
  $i \ne j$,
  \begin{align*}
    \Exv{\Lambda_i Y} &= \beta_i \gamma_i \\
    \Exv{\Lambda_i^2} &= \beta_i \\
    \Var{\Lambda_i Y} &= \beta_i - \beta_i^2 \gamma_i^2 \\
    \Var{\Lambda_i^2} &= \beta_i - \beta_i^2 \\
    \Cov{\Lambda_i Y, \Lambda_j Y} &= 0 \\
    \Cov{\Lambda_i^2, \Lambda_j^2} &= 0 \\
    \Cov{\Lambda_i Y, \Lambda_j^2} &= 0.
  \end{align*}
\end{restatable}
We also prove the following basic lemma that relates $\psi_i$ to $\gamma_i$.
\begin{restatable}{rlemma}{lemmagammaPsi}
  \label{lemmagammaPsi}
  It holds that
  \[
    \gamma_i = \tanh(\psi_i).
  \]
\end{restatable}
We also make the following claim about feasible models.
\begin{restatable}{rlemma}{lemmaConcentrationLambdaY}
  \label{lemmaConcentrationLambdaY}
  For any feasible model, it will be the case that, for any other feasible 
  parameter vector $\hat \psi$,
  \[
    \Prob{
      \hat \psi^T \Lambda Y
      \le
      \frac{m}{2} \gamma_{\min} (\gamma \beta)_{\min}
    }
    \le
    \exp\left(
      -
      \frac{
        m (\gamma \beta)_{\min} \gamma_{\min}^2
      }{
        9.34 \artanh(\gamma_{\max})
      }
    \right).
  \]
\end{restatable}
We can also prove the following simple result about the conditional
covariances
\begin{restatable}{rlemma}{lemmaSuffStatsMomentsCond}
  \label{lemmaSuffStatsMomentsCond}
  The covariances of the sufficient statistics, conditioned
  on $\Lambda$, are for all $i \ne j$,
  \begin{align*}
    \Covc{\Lambda_i Y, \Lambda_j Y}{\Lambda} &= \Lambda_i \Lambda_j \sech^2(\psi^T \Lambda) \\
    \Covc{\Lambda_i^2, \Lambda_j^2}{\Lambda} &= 0.
  \end{align*}
\end{restatable}
We can combine these two results to bound the expected variance of these
conditional statistics.
\begin{restatable}{rlemma}{lemmaExpectedVarCondBound}
  \label{lemmaExpectedVarCondBound}
  If $\theta$ and $\theta^*$ are two feasible models, then for any $u$,
  \[
    \Exv[\theta^*]{\Varc[\theta]{Y}{\Lambda}}
    \le
    3
    \exp\left(
      -
      \frac{
        m \beta_{\min}^2 \gamma_{\min}^3
      }{
        8 \artanh(\gamma_{\max})
      }
    \right).
  \]
\end{restatable}

We can now proceed to restate and prove the main corollary of Theorem
\ref{thmDataProgramming} that applies in the independent case.

\corDataProgrammingIndep*

\begin{proof}
  In order to apply Theorem \ref{thmDataProgramming}, we have to verify
  all its conditions hold in the independent case.

  First, we notice that (\ref{eqnDataProgAssumption1}) is used only to bound
  the covariance of the sufficient statistics.   From Lemma
  \ref{lemmaSuffStatsMoments}, we know that these can be bounded by
  $\beta_i - \beta_i^2 \gamma_i^2 \ge \frac{\beta_\min}{2}$.  It follows
  that we can choose
  \[
    c = \frac{\beta_{\min}}{4},
  \]
  and we can consider (\ref{eqnDataProgAssumption1}) satisfied, for the 
  purposes of applying the theorem.

  Second, to verify (\ref{eqnDataProgAssumption2}), we can use Lemma
  \ref{lemmaExpectedVarCondBound}.  For this to work, we need
  \[
    3
    \exp\left(
      -
      \frac{
        m (\gamma \beta)_{\min} \gamma_{\min}^2
      }{
        9.34 \artanh(\gamma_{\max})
      }
    \right)
    \le
    \frac{c}{M}
    =
    \frac{\beta_{\min}}{8 m}.
  \]
  This happens whenever the number of labeling functions satisfies
  \[
    m
    \ge
    \frac{
      9.34 \artanh(\gamma_{\max})
    }{
      (\gamma \beta)_{\min} \gamma_{\min}^2
    }
    \log\left(
      \frac{24 m}{\beta_{\min}}
    \right).
  \]

  The remaining assumptions, (\ref{eqnDataProgAssumption3}),
  (\ref{eqnDataProgAssumption4}),
  (\ref{eqnDataProgAssumption5}), and
  (\ref{eqnDataProgAssumption6}), are satisfied directly by the 
  assumptions of this corollary.  So, we can apply Theorem
  \ref{thmDataProgramming}, which produces the desired result.
\end{proof}

\section{Proofs of Independent Model Lemmas}

\lemmaSuffStatsMoments*

\begin{proof}
  We prove each of the statements in turn.  For the first statement,
  \begin{align*}
    \Exv{\Lambda_i Y}
    &=
    \Prob{\Lambda_i = Y}
    -
    \Prob{\Lambda_i = -Y}
    \\&=
    \beta_i \frac{1 + \gamma_i}{2} - \beta_i \frac{1 - \gamma_i}{2}
    \\&=
    \beta_i \gamma_i.
  \end{align*}
  For the second statement,
  \begin{align*}
    \Exv{\Lambda_i^2}
    &=
    \Prob{\Lambda = Y}
    +
    \Prob{\Lambda = -Y}
    \\&=
    \beta_i \frac{1 + \gamma_i}{2} + \beta_i \frac{1 - \gamma_i}{2}
    \\&=
    \beta_i.
  \end{align*}
  For the remaining statements, we derive the second moments; converting these
  to an expression of the covariance is trivial.  For the third statement,
  \[
    \Exv{(\Lambda_i Y)^2}
    =
    \Exv{\Lambda_i^2 Y^2}
    =
    \Exv{\Lambda_i^2}
    =
    \beta_i.
  \]
  For the fourth statement,
  \[
    \Exv{(\Lambda_i^2)^2}
    =
    \Exv{\Lambda_i^4}
    =
    \Exv{\Lambda_i^2}
    =
    \beta_i.
  \]
  For subsequent statements, we first derive that
  \[
    \Exvc{\Lambda_i Y}{Y}
    =
    \beta_i \frac{1 + \gamma_i}{2}
    -
    \beta_i \frac{1 - \gamma_i}{2}
    =
    \beta_i \gamma_i
  \]
  and
  \[
    \Exvc{\Lambda_i^2}{Y}
    =
    \beta_i \frac{1 + \gamma_i}{2}
    +
    \beta_i \frac{1 - \gamma_i}{2}
    =
    \beta_i.
  \]
  Now, for the fifth statement,
  \[
    \Exv{(\Lambda_i Y)(\Lambda_j Y)}
    =
    \Exv{\Exvc{\Lambda_i Y}{Y} \Exvc{\Lambda_j Y}{Y}}
    =
    \beta_i \gamma_i \beta_j \gamma_j.
  \]
  For the sixth statement,
  \[
    \Exv{(\Lambda_i^2)(\Lambda_j^2)}
    =
    \Exv{\Exvc{\Lambda_i^2}{Y} \Exvc{\Lambda_i^2}{Y}}
    =
    \beta_i \beta_j.
  \]
  Finally, for the seventh statement,
  \[
    \Exv{(\Lambda_i Y)(\Lambda_j^2)}
    =
    \Exv{\Exvc{\Lambda_i Y}{Y} \Exvc{\Lambda_i^2}{Y}}
    =
    \beta_i \gamma_i \beta_j.
  \]
  This completes the proof.
\end{proof}

\lemmagammaPsi*

\begin{proof}
  From the definitions,
  \[
    \beta_i
    =
    \frac{
      \exp(\psi_i + \phi_i)
      +
      \exp(-\psi_i + \phi_i)
    }{
      \exp(\psi_i + \phi_i)
      +
      \exp(-\psi_i + \phi_i)
      +
      1
    }
  \]
  and
  \[
    \beta_i \gamma_i
    =
    \frac{
      \exp(\psi_i + \phi_i)
      -
      \exp(-\psi_i + \phi_i)
    }{
      \exp(\psi_i + \phi_i)
      +
      \exp(-\psi_i + \phi_i)
      +
      1
    }.
  \]
  Therefore,
  \[
    \gamma_i
    =
    \frac{
      \exp(\psi_i + \phi_i)
      -
      \exp(-\psi_i + \phi_i)
    }{
      \exp(\psi_i + \phi_i)
      +
      \exp(-\psi_i + \phi_i)
    }
    =
    \tanh(\psi_i),
  \]
  which is the desired result.
\end{proof}

\lemmaConcentrationLambdaY*

\begin{proof}
  We start by noticing that
  \[
    \hat \psi^T \Lambda Y
    =
    \sum_{i=1}^m \hat \psi_i \Lambda_i Y.
  \]
  Since in this model, all the $\Lambda_i Y$ are independent of each other,
  we can bound this sum using a concentration bound.  First, we note that
  \[
    \Abs{\hat \psi_i \Lambda_i Y} \le \hat \psi_i.
  \]
  Second, we note that
  \[
    \Exv{\hat \psi_i \Lambda_i Y} = \hat \psi_i \beta_i \gamma_i
  \]
  and
  \[
    \Var{\hat \psi_i \Lambda_i Y}
    =
    \hat \psi_i^2 \left( \beta_i - \beta_i^2 \gamma_i^2 \right)
  \]
  but
  \[
    \Abs{ \hat \psi_i \Lambda_i Y }
    \le
    \hat \psi_i
    \le
    \artanh(\gamma_{\max})
    \triangleq
    \hat \psi_{\max}
  \]
  because, for feasible models, by definition
  \[
    \gamma_{\min}
    \le
    \artanh(\gamma_{\min})
    \le
    \hat \psi_i
    \le
    \artanh(\gamma_{\max}).
  \]
  Therefore, applying Bernstein's inequality gives us, for any $t$,
  \[
    \Prob{
      \sum_{i=1}^m \hat \psi_i \Lambda_i Y
      -
      \sum_{i=1}^m \hat \psi_i \beta_i \gamma_i
      \le
      -t
    }
    \le
    \exp\left(
      -
      \frac{
        3 t^2
      }{
        6 \sum_{i=1}^m \hat \psi_i^2 \gamma_i \beta_i \gamma_i
        +
        2 \hat \psi_{\max} t
      }
    \right).
  \]
  It follows that, if we let
  \[
    t = \frac{1}{2} \sum_{i=1}^m \hat \psi_i \beta_i \gamma_i,
  \]
  then we get
  \begin{align*}
    \Prob{
      \sum_{i=1}^m \hat \psi_i \Lambda_i Y
      -
      \sum_{i=1}^m \hat \psi_i \beta_i \gamma_i
      \le
      -t
    }
    &\le
    \exp\left(
      -
      \frac{
        3
        \left(
          \frac{1}{2} \sum_{i=1}^m \hat \psi_i \beta_i \gamma_i
        \right)^2
      }{
        6 \sum_{i=1}^m \hat \psi_i^2 \gamma_i \beta_i \gamma_i
        +
        2 \hat \psi_{\max}
        \left(
          \frac{1}{2} \sum_{i=1}^m \hat \psi_i \beta_i \gamma_i
        \right)
      }
    \right)
    \\ &\le
    \exp\left(
      -
      \frac{
        3 \sum_{i=1}^m \hat \psi_i \beta_i \gamma_i
      }{
        24 \gamma_{\max} \hat \psi_{\max}
        +
        4 \hat \psi_{\max}
      }
    \right)
    \\ &\le
    \exp\left(
      -
      \frac{
        3 m (1 - \gamma_{\max})
      }{
        28 \hat \psi_{\max}
      }
    \right)
    \\ &\le
    \exp\left(
      -
      \frac{
        3 \left( \sum_{i=1}^m \hat \psi_i \beta_i \gamma_i \right)^2
      }{
        24 \sum_{i=1}^m \hat \psi_i^2 \beta_i
        +
        4 \hat \psi_{\max}
        \left(
          \sum_{i=1}^m \hat \psi_i \beta_i \gamma_i
        \right)
      }
    \right)
    \\ &\le
    \exp\left(
      -
      \frac{
        3 \gamma_{\min} \left( \sum_{i=1}^m \hat \psi_i \beta_i \right)
        \left( \sum_{i=1}^m \hat \psi_i \beta_i \gamma_i \right)
      }{
        24 \hat \psi_{\max} \sum_{i=1}^m \hat \psi_i \beta_i
        +
        4 \hat \psi_{\max}
        \left(
          \sum_{i=1}^m \hat \psi_i \beta_i
        \right)
      }
    \right)
    \\ &\le
    \exp\left(
      -
      \frac{
        3 \gamma_{\min} \left( \sum_{i=1}^m \hat \psi_i \beta_i \gamma_i \right)
      }{
        28 \hat \psi_{\max}
      }
    \right)
    \\ &\le
    \exp\left(
      -
      \frac{
        m \gamma_{\min}^2 (\gamma \beta)_{\min}
      }{
        9.34 \hat \psi_{\max}
      }
    \right).
  \end{align*}
  This is the desired expression.
\end{proof}

\lemmaSuffStatsMomentsCond*

\begin{proof}
  The second result is obvious, so it suffices to prove only the first
  result.  Clearly,
  \[
    \Covc{\Lambda_i Y, \Lambda_j Y}{\Lambda}
    =
    \Lambda_i \Lambda_j
    \Varc{Y}{\Lambda}
    =
    \Lambda_i \Lambda_j
    \left(
      1 - \Exvc{Y}{\Lambda}^2
    \right).
  \]
  Plugging into the distribution formula lets us conclude that
  \[
    \Exvc{Y}{\Lambda}
    =
    \frac{
      \exp(\psi^T \Lambda + \phi^T \Lambda^2)
      -
      \exp(-\psi^T \Lambda + \phi^T \Lambda^2)
    }{
      \exp(\psi^T \Lambda + \phi^T \Lambda^2)
      +
      \exp(-\psi^T \Lambda + \phi^T \Lambda^2)
    }
    =
    \tanh^2(\psi^T \Lambda),
  \]
  and so
  \[
    \Covc{\Lambda_i Y, \Lambda_j Y}{\Lambda}
    =
    \Lambda_i \Lambda_j
    \left(
      1 - \tanh^2(\psi^T \Lambda)
    \right)
    =
    \Lambda_i \Lambda_j \sech^2(\psi^T \Lambda),
  \]
  which is the desired result.
\end{proof}

\lemmaExpectedVarCondBound*

\begin{proof}
  First, we note that,  by the result of
  Lemma \ref{lemmaSuffStatsMomentsCond},
  \begin{align*}
    \Varc[\theta]{Y}{\Lambda}
    &=
    \sech^2(\psi^T \Lambda).
  \end{align*}
  Therefore,
  \[
    \Exv[\theta^*]{\Varc[\theta]{Y}{\Lambda}}
    =
    \Exv[\theta^*]{ \sech^2(\psi^T \Lambda) }.
  \]
  Applying Lemma \ref{lemmaConcentrationLambdaY}, we can bound this with
  \begin{align*}
    \Exv[\theta^*]{\Varc[\theta]{u^T \Lambda Y}{\Lambda}}
    &\le
    \left(
      \sech^2\left( \frac{m}{2} (\gamma \beta)_{\min} \gamma_{\min}^2 \right)
      +
      \exp\left(
        -
        \frac{
          m (\gamma \beta)_{\min} \gamma_{\min}^2
        }{
          9.34 \artanh(\gamma_{\max})
        }
      \right)
    \right) \\
    &\le
    \left(
      2 \exp\left( -\frac{m}{2} (\gamma \beta)_{\min} \gamma_{\min}^2 \right)
      +
      \exp\left(
        -
        \frac{
          m (\gamma \beta)_{\min} \gamma_{\min}^2
        }{
          9.34 \artanh(\gamma_{\max})
        }
      \right)
    \right) \\
    &\le
    3
    \exp\left(
      -
      \frac{
        m (\gamma \beta)_{\min} \gamma_{\min}^2
      }{
        9.34 \artanh(\gamma_{\max})
      }
    \right).
  \end{align*}
  This is the desired expression.
\end{proof}

\section{Additional Experimental Details}

\subsection{Relation Extraction Experiments}
\label{sec:relexps}
\subsubsection{Systems}
The original distantly-supervised experiments which we compare against as baselines--which we refer to as using the \textit{if-then-return (ITR)} approach of distant or programmatic supervision--were implemented using DeepDive, an open-source system for building extraction systems.\footnote{\small{\url{http://deepdive.stanford.edu}}} For our primary experiments,
we adapted these programs to the framework and approach described in this paper, directly utilizing distant supervision rules as labeling functions.

In the disease tagging user experiments, we used an early version of our new lightweight extraction framework based around data programming, formerly called DDLite~\cite{ehrenberg2016data}, now Snorkel.\footnote{\small{\url{http://snorkel.stanford.edu}}}
Snorkel is based around a Jupyter-notebook based interface, allowing users to iteratively develop labeling functions in Python for basic extraction tasks involving simple models. Details of the basic discriminative models used can be found in the Snorkel repository; in particular, Snorkel uses a simple logistic regression model with generic features defined in part over dependency paths\footnote{\small{\url{https://github.com/HazyResearch/treedlib}}}, and a basic LSTM  model implemented using the Theano library.\footnote{\small{\url{http://deeplearning.net/software/theano/}}}
Snorkel is currently under continued development, and all versions are open-source.

\subsubsection{Applications}
We consider three primary applications which involve the extraction of 
binary relation mentions of some specific type from unstructured text input data.
At a high level, all three system pipelines 
consist of an initial \textit{candidate extraction} phase which leverages some 
upstream model or suite of models to extract mentions of involved entities, and then 
considers each pair of such mentions that occurs within the same local neighborhood 
in a document as a \textit{candidate relation mention} to be potentially extracted.
In each case, the discriminative model that we are aiming to train--and that we evaluate in this paper--is a binary 
classifier over these candidate relation mentions, which will decide which ones to output as final true extractions.  In all tasks, we preprocessed raw input text with Stanford CoreNLP\footnote{\small{\url{stanfordnlp.github.io/CoreNLP/}}}, and then either used CoreNLP's NER module or our own entity-extraction models to extract entity mentions. Further details of the basic information extraction pipeline utilized can be seen in the tutorials of the systems used, and in the referenced papers below.

In the 2014 TAC-KBP Slot Filling task, which we also refer to as the News application, we train a set of extraction models for a variety of relation types from news articles~\cite{surdeanu2014overview}. In reported results in this paper, we average over scores from each relation type. We utilized CoreNLP's NER module for candidate extraction, and utilized CoreNLP outputs in developing the distant supervision rules / labeling functions for these tasks.  We also considered a slightly simpler discriminative model than the one submitted in the 2014 competition, as reported in \cite{angeli2014stanford}: namely, we did not include any joint factors in our model in this paper.

In the Genomics application, our goal with our collaborators at Stanford Medicine was to extract mentions of genes that if mutated may cause certain phenotypes (symptoms) linked to Mendelian diseases, for use in a clinical diagnostic setting. The code for this project is online, although it remains partially under development and thus some material from our collaborators is private.\footnote{\small{\url{https://github.com/HazyResearch/dd-genomics}}}

In the Pharmacogenomics application, our goal was to extract interactions between genes for use in downstream pharmacogenomics research analyses; full results and system details are reported in \cite{mallory2015large}.

In the Disease Tagging application, which we had our collaborators work on during a set of short hackathons as a user study, the goal was to tag mentions of human diseases in PubMed abstracts.  We report results of this hackathon in \cite{ehrenberg2016data}, as well as in our Snorkel tutorial online.

\subsubsection{Labeling Functions}
In general, we saw two broad types of labeling functions in both prior applications (when they were referred to as ``distant supervision rules'') and in our most recent user studies.  The first type of labeling function leverages some weak supervision signal, such as an external knowledgebase (as in traditional distant supervision), very similar to the example illustrated in Fig. 1(a).  All of the applications studied in this paper used some such labeling function or set of labeling functions.

The second type of labeling function uses simple heuristic patterns as positive or negative signals. For our text extraction examples, these heuristic patterns primarily consisted of regular expressions, also similar to the example pseudocode in Fig. 1(a).  Further specific details of both types of labeling functions, as well as others used, can be seen in the linked code repositories and referenced papers.

\subsection{Synthetic Experiments}
\begin{figure}
\centering
\begin{subfigure}{.33\textwidth}
  \centering
  \includegraphics[width=1.0\linewidth]{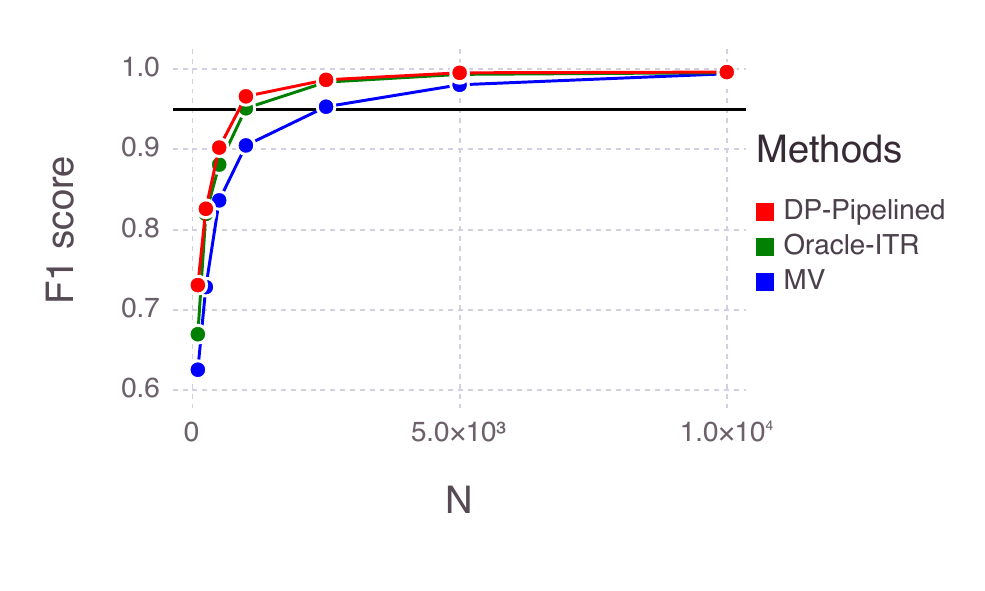}
  \caption{$m=20$}
  \label{fig:sub1}
\end{subfigure}%
\begin{subfigure}{.33\textwidth}
  \centering
  \includegraphics[width=1.0\linewidth]{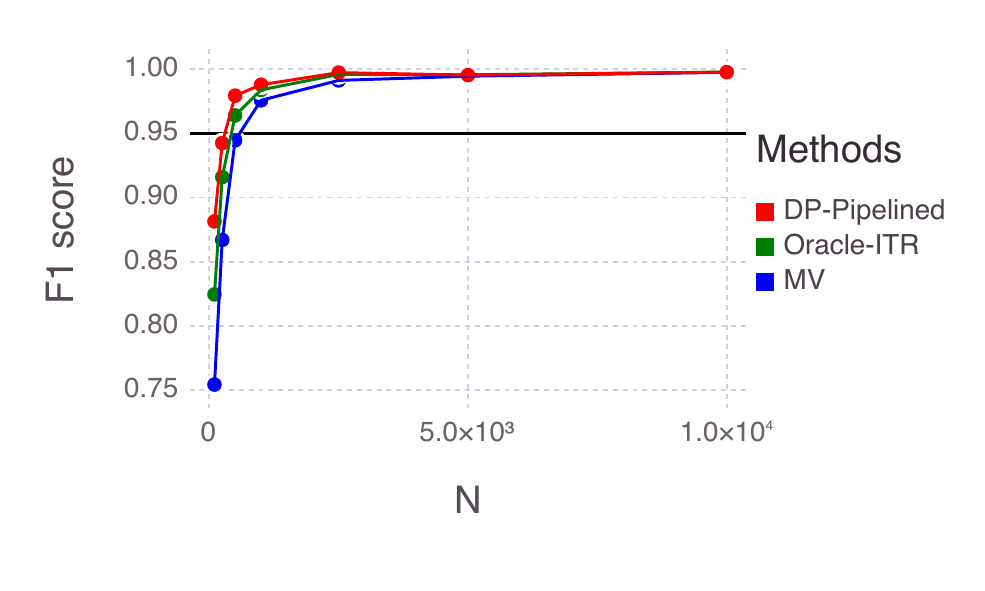}
  \caption{$m=100$}
  \label{fig:sub2}
\end{subfigure}
\begin{subfigure}{.33\textwidth}
  \centering
  \includegraphics[width=1.0\linewidth]{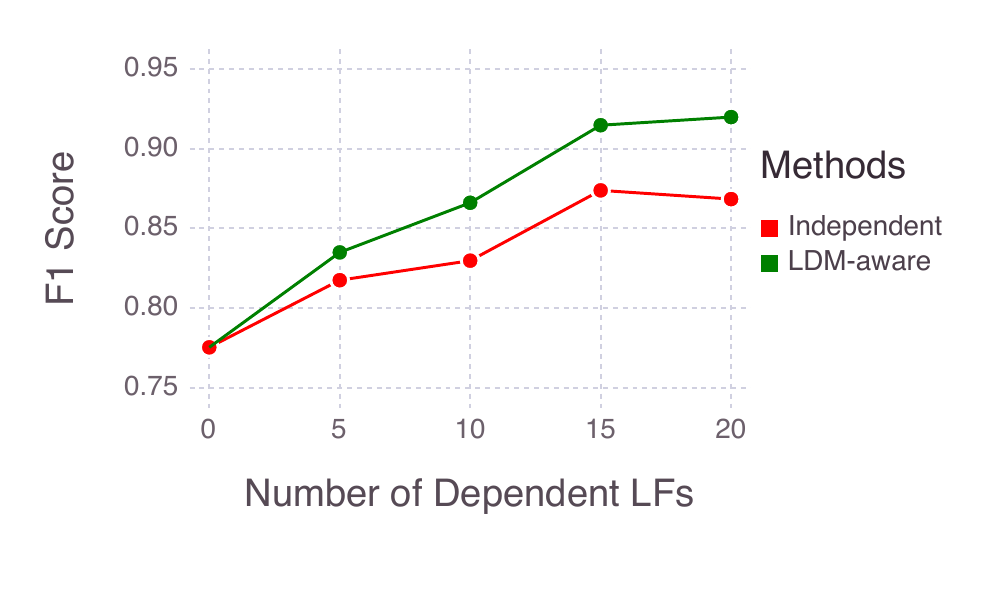}
  \caption{Adding dependencies.}
  \label{fig:sub3}
\end{subfigure}
\caption{Comparisons of data programming to two oracle methods on synthetic data.}
\label{synthetics}
\end{figure}
In Fig. 3(a-b), we ran synthetic experiments with labeling functions having constant coverage $\beta=0.1$,
and accuracy drawn from $\alpha\sim\textrm{Uniform}(\mu_\alpha-0.25,\mu_\alpha+0.25)$ where $\mu_\alpha=0.75$ in the
above plots.  In both cases we used $1000$ normally-drawn features having mean correlation with the true label class of 0.5.

In this case we compare data programming (DP-Pipelined) against two baselines.  First, we compare against an \textit{if-then-return}
setup where the ordering is optimal (ITR-Oracle).  Second, we compare against simple majority vote (MV).

In Fig. 3(c), we show an experiment where we add dependent labeling functions to a set of $m_{ind}=50$ independent
labeling functions, and either provided this dependency structure (LDM-Aware) or did not (Independent).  In this case,
the independent labeling functions had the same configurations as in (a-b), and the dependent labeling functions corresponded
to ``fixes'' or ``reinforces''-type dependent labeling functions.

\end{document}